\documentclass[[10pt,journal,compsoc]{IEEEtran}
\usepackage{cite}
\ifCLASSINFOpdf
  \usepackage[pdftex]{graphicx}
\else
\fi
\usepackage{amsmath,amssymb,amsfonts,amsthm}
\usepackage{algorithm}
\usepackage{algpseudocode}
\newtheorem{assumption}{Assumption}[section]
\newtheorem{theorem}{Theorem}[section]

\algblockdefx{ForAllP}{EndForAllP}[1]%
  {\textbf{for all }#1 \textbf{do in parallel}}%
  {\textbf{end for}}
  
\algnewcommand\algorithmicpsdo{\textbf{Central server do:}}
\algnewcommand\PSDO{\item[\algorithmicpsdo]}
\algnewcommand\algorithmicclientdo{\textbf{Clients do:}}
\algnewcommand\ClientDO{\item[\algorithmicclientdo]}
\algnewcommand\algorithmicblank{\textbf{}}
\algnewcommand\Blank{\item[\algorithmicblank]}

\algnewcommand\algorithmictrainlocally{\textbf{TrainLocally($k$, $w_0$):}}
\algnewcommand\TrainLocally{\item[\algorithmictrainlocally]}

\ifCLASSOPTIONcompsoc
 \usepackage[caption=false,font=normalsize,labelfont=sf,textfont=sf]{subfig}
\else
 \usepackage[caption=false,font=footnotesize]{subfig}
\fi

\usepackage{stfloats}
\usepackage{url}
\usepackage{caption}
\usepackage{tabularx}
\usepackage{multirow}
\usepackage{booktabs}
\usepackage{mathtools}
\usepackage{color,soul}

\begin{document}
\title{Communication-Efficient Federated Learning with Compensated Overlap-FedAvg}

\author{Yuhao~Zhou$^{\ast}$,~Qing~Ye$^{\ast}$,~and~Jiancheng Lv,~\IEEEmembership{Member,~IEEE}% <-this % stops a space
\thanks{This work is supported in part by the National Key Research and Development Program of China under Contract 2017YFB1002201, in part by the National Natural Science Fund for Distinguished Young Scholar under Grant 61625204, and in part by the State Key Program of the National Science Foundation of China under Grant 61836006.}% <-this % stops a space
\thanks{Yuhao~Zhou~(e-mail:~sooptq@gmail.com), Qing~Ye~(e-mail:~fuyeking@stu.scu.edu.cn), and Jiancheng~Lv~(e-mail:~lvjiancheng@scu.edu.cn) are with the College of Computer Science, Sichuan University, China.}% <-this % stops a space
\thanks{$^{\ast}$ indicates that Yuhao Zhou and Qing Ye contributed equally to this paper.}% <-this % stops a space
}


\markboth{Journal of \LaTeX\ Class Files,~Vol.~14, No.~8, August~2015}%
{Zhou \MakeLowercase{\textit{et al.}}: Communication-Efficient Federated Learning with Compensated Overlap-FedAvg}



\IEEEtitleabstractindextext{%
\begin{abstract}
While petabytes of data are generated each day by a number of independent computing devices, only a few of them can be finally collected and used for deep learning (DL) due to the apprehension of data security and privacy leakage, thus seriously retarding the extension of DL. In such a circumstance, federated learning (FL) was proposed to perform model training by multiple clients' combined data without the dataset sharing within the cluster. Nevertheless, federated learning with periodic model averaging (FedAvg) introduced massive communication overhead as the synchronized data in each iteration is about the same size as the model, and thereby leading to a low communication efficiency. Consequently, variant proposals focusing on the communication rounds reduction and data compression were proposed to decrease the communication overhead of FL. In this paper, we propose Overlap-FedAvg, an innovative framework that loosed the chain-like constraint of federated learning and paralleled the model training phase with the model communication phase (i.e., uploading local models and downloading the global model), so that the latter phase could be totally covered by the former phase. Compared to vanilla FedAvg, Overlap-FedAvg was further developed with a hierarchical computing strategy, a data compensation mechanism, and a nesterov accelerated gradients (NAG) algorithm. In Particular, Overlap-FedAvg is orthogonal to many other compression methods so that they could be applied together to maximize the utilization of the cluster. Besides, the theoretical analysis is provided to prove the convergence of the proposed framework. Extensive experiments conducting on both image classification and natural language processing tasks with multiple models and datasets also demonstrate that the proposed framework substantially reduced the communication overhead and boosted the federated learning process.
\end{abstract}

}

\maketitle



% \fi
\IEEEpeerreviewmaketitle

\section{Introduction}
\IEEEPARstart{W}ith the rapid development of deep learning, tons of daily generated data that were previously considered useless can now be utilized to extract latent patterns through deep learning (DL)~\cite{Lecun2015Deep}, and thereby retrieving valuable information. Similarly, owning abundant data resources is a prerequisite for many innovative breakthroughs in terms of academia across different fields, including Natural Language Processing\cite{Brown2020LanguageMA, devlin2018bert} and Computer Vision\cite{krizhevsky2017imagenet, simonyan2014very}. On the other hand, thanks to the popularity of independent computing devices~\cite{edge-computing} (e.g. smartphone, laptop) and edge devices (e.g. router), each individual now can produce more data than ever before, with gigabytes or even terabytes expected to be generated every day. However, these fresh data are tied to different facilities, resulting in data fragmentation (i.e. isolated data islands problem). In other words, since separated facilities like smartphones are mostly equipped with lots of sensors (e.g. camera, microphone, GPS), and common users generally have no idea either what applications are collecting information through these sensors or how this collected information would be used for, these users will take considerable risks when they decide to share this information with others. As a result, while the demand for data and the speed of creating data have both been increased dramatically, it is still extremely difficult to gather these data together, processing, and making the most of them into a collection for learning, especially considering privacy is increasingly valued.

To tackle this problem, federated learning (FL)~\cite{mcmahan2017communication} was proposed to introduce a new promising centralized distributed training method~\cite{Demystifying-DDL} that allows multiple clients (e.g. mobile clients or servers from different enterprises) to coordinately train a deep neural network model on their combined data, without any of the participants having to reveal its local data to the central server or other participants. Specifically, the workflow of the federated learning with FederatedAveraging (i.e., FedAvg~\cite{mcmahan2017communication}), as Figure~\ref{fig:general-fl} illustrated, has 4 steps: I. The central server initializes the global model and pushes the global model to every participating client. II. Clients train the received global model on its local data. After that, they return the essential information $I_t^k$ that is beneficial to the evolving of the global model (i.e., mostly is model weights, but could be gradients) to the server, where $k$ is the index of the client and the $t$ indicates the iteration index, respectively. III. The central server utilizes all clients' information by $\sum_k I_t^k$ to update the global model. IV. Repeat step II and step III until convergence. Particularly, it can be seen that federated learning omits the data collection step, and conversely utilizing $I_t^k$ that is not explicitly related to clients' local data to evolve the global model, which not only integrates fragmented data, but also offers as much privacy as possible.

\begin{figure}[htb]
	\centering
	\includegraphics[width=\linewidth]{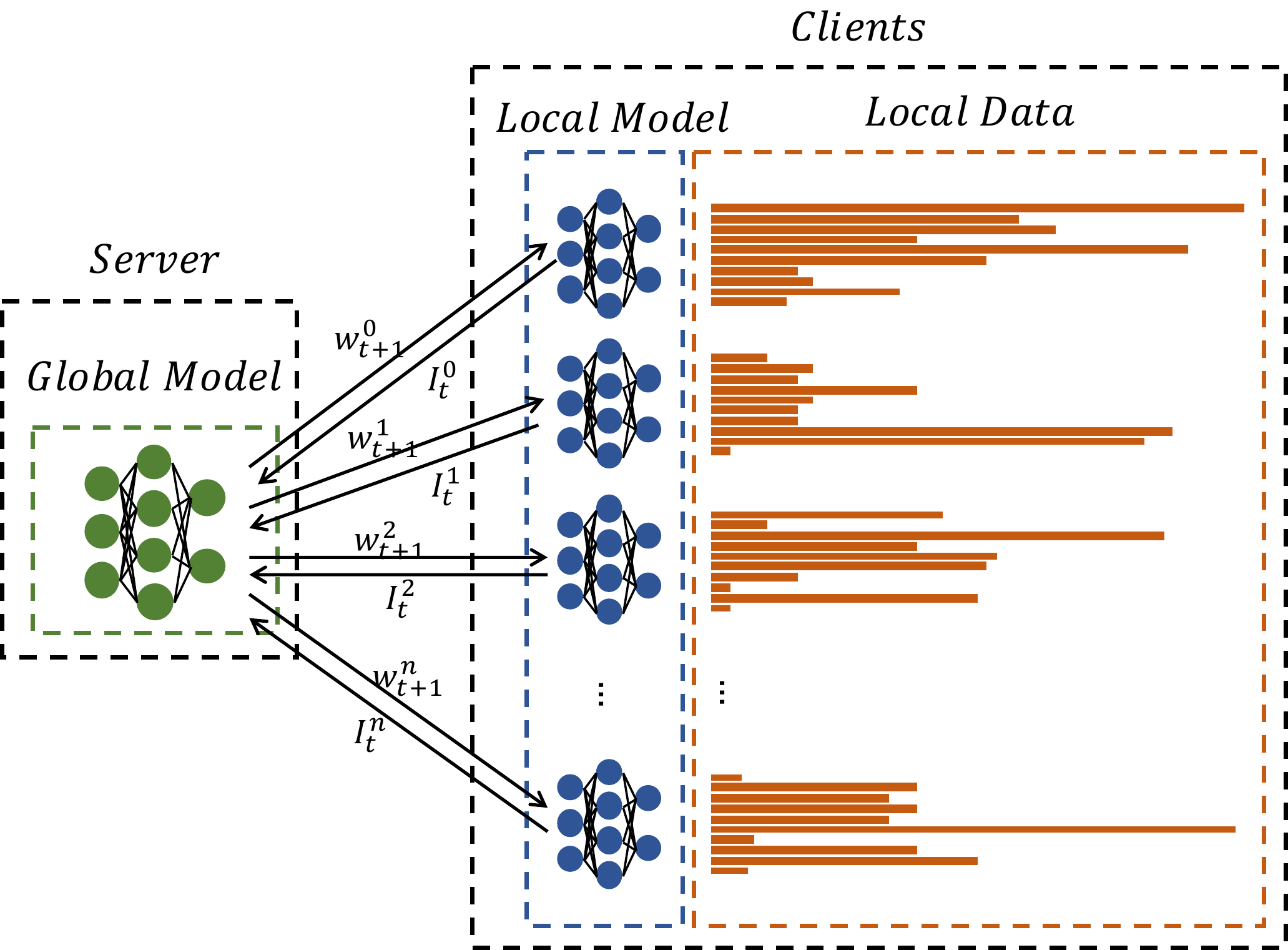}
	\caption{A general federated learning workflow. Each client always keeps its data private and only shares essential information that could help improve the global model with others, be it gradients or model weights.}
	\label{fig:general-fl}
\end{figure}

Nonetheless, preserving data privacy by not revealing each client's local data in federated learning comes at a price, especially considering the fact that deep learning is basically built on top of huge chunks of data. In such a circumstance, the price for federated learning~\cite{2019arXiv191204977K} is low communication efficiency due to overloaded bandwidth since the client-server communication occurs at the end of every training iteration. In detail, referring to the workflow of federated learning as shown in Figure~\ref{fig:general-fl}, we can see that at each iteration, every client has to communicate a $I_t^k$ which commonly has the same size as the model. Given that nowadays the model size can be easily reached hundreds of megabytes~\cite{Brown2020LanguageMA, huang2017densely, xie2017aggregated} or even gigabytes, plus the model in the federated learning generally takes hundreds of thousands of iterations to converge, the overall data transfer size during the training process may exceed petabytes readily\cite{sattler2019sparse}. Hence, in comparison to non-privacy-preserving training, federated learning usually takes much more time to converge due to the time-consuming communication phase, and it may become even slower when the network condition is relatively poorer.

To reduce the communication cost of federated learning, various methods were proposed to explore the possible solutions, which can be mainly divided into two different categories in this paper. Specifically, the first category is the model averaging with a large communication interval. Some work~\cite{mcmahan2017communication, konevcny2016federated} tried to drastically decrease the communication overhead by enlarging the communication interval, yet it also degraded the model's final accuracy and the best communication step size is hard to capture. The second category is data compression. Some work~\cite{seide20141, bernstein2018signsgd, alistarh2017qsgd, lin2017deep, wangni2018gradient} attempted to compress the $I_t^k$ before transmitting it. However, the compressing phase needs to consume a lot of time as well, especially when performing federated learning on battery-sensitive (e.g. smartphone) or low-performance (e.g. netbook, gateway) devices.

In this paper, we propose a new novel framework named Overlap-FedAvg to alleviate this problem, which paralleled the model training phase with the model uploading and downloading phase, so that the latter phase could be totally covered by the former phase. Compared to the vanilla FedAvg, Overlap-FedAvg was firstly further developed with a hierarchical computing strategy to accomplish the parallelism. Nevertheless, this strategy also brought the gradient staleness problem so a gradient compensation mechanism was therefore employed to alleviate this issue and keep the convergence of the proposed Overlap-FedAvg framework. Additionally, a nesterov accelerated gradients (NAG)~\cite{NAG, yang2020federated} algorithm was also established to accelerate the convergence speed of the federated learning. Besides, Overlap-FedAvg is orthogonal to other compression methods so that it could collaborate with others to further improve the efficiency of the federated learning.

The contributions of this proposal are the followings:

\begin{itemize}
	\item To improve the communication efficiency of federated learning, we propose a novel communication-efficient framework called Overlap-FedAvg by paralleling the communication with computation. To the best of our knowledge, it is the first work that focuses on improving the communication efficiency of federated learning by an overlapping approach, and it is totally compatible with many other data compression methods. 
	\item In order to alleviate the gradients staleness problem brought by the parallelism in Overlap-FedAvg and keep the convergence speed of the federated learning, taylor series expansion was utilized as a compensation mechanism. Additionally, the essential theoretical analysis of the proposed Overlap-FedAvg algorithm is provided.
	\item To achieve higher accuracy, NAG was employed to further accelerate the federated learning. Compared with previous similar work, our proposal was not only effective but also intuitive and easy to implement. Note that theoretically, not only the NAG algorithm but all acceleration algorithms (e.g. Adam\cite{kingma2014adam}, RMSProp\cite{Tieleman2012}, etc.) could be applied to the Overlap-FedAvg framework readily.
	\item To investigate the effectiveness of the proposed Overlap-FedAvg, extensive experiments were conducted on four benchmark datasets with 8 DNN models to compare its potentiality with the baseline algorithm FedAvg. The results demonstrated that our proposed Overlap-FedAvg could effectively improve the accuracy of federated learning and decrease the communication cost. The source code and hyper-parameters of all experiments are open-sourced for reproducibility\footnote{\url{https://github.com/Soptq/Overlap-FedAvg}}. In addition, we discuss whether the proposed Overlap-FedAvg would damage the privacy security of federated learning.
\end{itemize}

The rest of this paper is organized as follows. A literature review is illustrated in Section~\ref{section:two}, where some background information and related work are introduced. In Section~\ref{section:three}, the implementations of the proposed Overlap-FedAvg algorithm are presented. Then the essential convergence analysis is described in Section~\ref{section:four}. The experiment design and result analysis are detailedly documented in Section~\ref{section:five}. Finally, the discussion and conclusion of this paper are drawn in Section~\ref{section:six}.

\section{Literature review}
\label{section:two}
In this section, we give a brief introduction of federated learning, and some related work that accelerates the training process of federated learning by improving communication efficiency.

\subsection{Federated Learning}
Federated learning generally describes a distributed framework that allows the cluster to perform model training on all clients' combined datasets without leaking their respective local data to others. This is accomplished by letting every participating client only sends information that could help evolve the global model, instead of its plainly local dataset, to the central server. Federated learning is originated from distributed deep learning~\cite{Large-scale-DDN,Ea-based-NAS, PSO-PS}, but with mainly four additional features. Firstly, the training data of different clients used in federated learning is expected to be heavily unbalanced. Namely, each client's local data tend to be both Non-IID (Independent and Identically Distributed) and unequal in amounts, and will not be unveiled to others during the training process. As a result, any node in the cluster, including the central server, cannot determine the data distribution of any other nodes. Secondly, federated learning has a relatively small training batch size. Virtually, federated learning was initially designed to make use of the data on battery-sensitive or low-performance devices, which usually also have a relatively small memory size. Consequently, the batch size of the model needs to be small enough so that the memory would not run out. Third, the network bandwidth of the clients in federated learning is always considerably small compared to the vanilla distributed training, where gigabytes bandwidth is commonly used. Hence, numerous communications produce massive communication overhead and take much time, resulting in low communication efficiency. Moreover, network connections of clients can be even lost during the training, indicating the system has to be robust enough to handle these exceptions. Forth, clients in federated learning are unreliable, meaning it is very potential that malicious clients send carefully constructed information to the central server to poison the global model, and eventually making the global model unusable. Extensive work was proposed to address the issues aforementioned~\cite{bonawitz2017practical, hardy2017private, abadi2016deep, zhao2018federated, sattler2019robust} while we focus on alleviating the massive communication overhead problem of federated learning in this article.

\begin{algorithm}[h] 
	\caption{\texttt{FederatedAveraging}. In the cluster there are $N$ clients in total, each with a learning rate of $\eta$. The set containing all clients is denoted as $S$, the communication interval is denoted as $E$, and the fraction of clients is denoted as $C$}
	\label{alg:fedavg} 
	\begin{algorithmic}[1] 
	    \PSDO
	        \State Initialization: global model $w_0$.
	        \For {each global iteration $t \in {1, ..., iteration}$}
	            \State \# Determine the number of participated clients.
	            \State $m \leftarrow max(C \cdot N, 1)$
	            \State \# Randomly choose participated clients.
	            \State $S_p = random.choice(S, m)$
	            \ForAllP {each client $k \in S_p$}
	                \State \# Get clients improved model.
	                \State $w_{t+1}^{k} \leftarrow TrainLocally(k, w_t)$
	            \EndForAllP
	            \State \# Update the global model.
	            \State $w_{t+1} \leftarrow \sum_{k=0}^{N} p_{k}w_{t+1}^{k}$
	        \EndFor
    \Blank
	\TrainLocally
	    \For {each client iteration $e \in {1, ..., E}$}
	        \State \# Do local model training.
	        \State $w_{e} \leftarrow w_{e-1} - \eta \nabla F(w_{e-1})$
	    \EndFor
	    \State \Return $w_{E}$
	\end{algorithmic} 
\end{algorithm}

\subsection{Related Work}
Various work had been proposed to address the massive communication overhead problem during federated learning, which can be roughly divided into two groups: \textbf{reducing communication rounds} and \textbf{data compression}

\begin{enumerate}
    \item \textbf{Reducing communication rounds:} In vanilla federated learning, the communication phase happened at the end of every iteration (i.e., the communication interval is $1$). Considering a typical federated DNN training process usually takes hundreds of thousands of iterations, one can easily think of enlarging the communication interval to significantly reduce the communication overhead. As a result, FederatedAveraging (FedAvg)~\cite{mcmahan2017communication} and its variants \cite{konevcny2016federated} were proposed to allow clients do multiple iterations of local training before making a global model update, as Algorithm~\ref{alg:fedavg} illustrated, where $p_k = \frac{n_k}{n}$ is the weight of the $k$-th client. The experiments suggested that FedAvg massively increased the convergence speed with respect to wall-clock time due to the reduction of communication rounds. In addition, several work provided the theoretical analysis of FedAvg, advocating that it is not only convergent on both IID and Non-IID data with \textit{decaying learning rate}, but also has linear speed-up \cite{li2019convergence, qu2020federated}. However, The communication interval in FedAvg is controlled by a hyper-parameter $E$, which is greatly influential to the final accuracy of the model and is shifty based on the characteristic of both the model and the dataset. In fact, the choice of $E$ is a trade-off between the final accuracy of DNN and the training efficiency: the smaller the $E$ is, the better the model's final accuracy generally would be, and conversely the bigger the $E$ is, the faster the model generally would converge. Therefore, an experienced engineer needs to be employed to fine-tune the communication interval $E$ in order to extract the best performance of the model.
    \item \textbf{Data compression:} Apart from decreasing communication rounds, another approach to reduce the communication overhead is to reduce the data transfer size, namely data compression. Following this direction, there are mainly two kinds of methods to effectively compress the data: quantization\cite{seide20141, bernstein2018signsgd, alistarh2017qsgd} and sparsification\cite{lin2017deep, wangni2018gradient}, where the former one aims to represent the original data by a low-precision data type with a smaller size (e.g. \texttt{int8} or \texttt{bool}), and the latter one intentionally only transmits essential values at each communication (about $1\%$ of the total number of values). In detail, quantization with error feedback is experimentally and theoretically effective\cite{karimireddy2019error}, but its compression rate is considerably low as the maximum compression ratio is limited to $\frac{1}{32}$ (the common data type used in deep learning is 32-bit). Moreover, with fewer bits carrying information, quantization tends to converge much slower. On the other hand, sparsification is capable of achieving a compression rate of $\frac{1}{100}$ easily without significantly damaging the model's convergence speed and final accuracy. Nevertheless, sparsification inevitably introduces extra phases during the training process, including sampling, compressing, encoding, decoding, and decompressing, which could potentially affect the overall training efficiency, especially on battery-sensitive (e.g. smartphone) or low-performance (e.g. netbook, gateway) devices.
\end{enumerate}

In this paper, we propose Overlap-FedAvg that attempted to overlap the communication with computation in federated learning to boost the training from a structural perspective. In fact, the idea of communication-computation overlap had been explored extensively in High-Performance Computing (HPC)~\cite{danalis2005transformations, quinn1996utility, marjanovic2010overlapping}. However, they generally assumed the data and the computation is independent of each other so that data can be transmitted as soon as it is available. Nonetheless, in the training process of a DNN, each iteration of training seriously related to the data of the last iteration, leading to a chain-like structure that is hard to split. Consequently, ideas in the HPC domain fail to be directly applied to federated learning. To accelerate the training of federated learning, the Overlap-FedAvg is proposed, which is the first overlapping work in federated learning to our best knowledge. Particularly, the dependency problem was addressed by relaxing the chain-like constraint and a data compensation mechanism, which will be detailedly discussed in the latter section.

Moreover, compared to above-mentioned related work and the corresponding shortcomings, Overlap-FedAvg firstly does not require experts to carefully design $E$. Contrarily, $E$ in the Overlap-FedAvg is automatically determined by the environment (i.e. network condition, bandwidth, etc) with the hierarchical computing strategy. In other words, if the environment permits a more frequent communication, $E$ will be automatically tuned smaller, and vise versa. Secondly, Overlap-FedAvg does not compress any data during the training process. That is, there is no reduced convergence rate or extra time-consuming phases brought by Overlap-FedAvg in terms of compression. Moreover, Overlap-FedAvg is compatible with those compression methods, so that they can be applied together to the training process to further accelerate the training speed.

\section{Methodology}
\label{section:three}

In this section, we firstly show an overview of the proposed Overlap-FedAvg framework, then we present its details as well as implementations, and finally we introduce the gradients compensation mechanism and the application of the NAG algorithm.

\subsection{Overview}

To alleviate the massive communication problem, we propose a novel framework named Overlap-FedAvg, which decoupled the model uploading and downloading phase with the model training phase by employing a separate process dedicating to data transmission on each client. In this way, clients' local model training phases can be freed from the interruptions of the frequent communication. However, while this decoupling made the training phase capable of doing continuous model training without any blocking, it also introduced staleness to the synchronized data, which will be detailedly explained in Section \ref{section:three:gradients-compensation}. Observed that the weights of the model are generally equivalent to the aggregation of the gradients, we decomposed the vanilla FedAvg's model updating rule, abstracting the gradients, utilizing taylor series expansion and fisher information matrix to analyze the gap between the stale gradients and the up-to-date gradients, and thereby properly compensating the stale model weights. Moreover, with the decomposition and the abstraction of the gradients, many parameters optimizing acceleration algorithms can be easily applied to further increase the converge speed of federated learning, for example, momentum accelerated SGD, NAG, or even Adam~\cite{kingma2014adam}.

In general, Overlap-FedAvg's workflow consists of five steps. I. Initiate the global model on the central server and push it to every participating client. II. Each participating client does $E$ iterations of local model training on their respective data. Note that $E$ here is no longer a constant hyper-parameter that needs to be manually setup, but a dynamic variable that adapts to the changing circumstances. III. When the local training is finished, clients instantly continue the next iteration of local training, while commanding another process to push their local model to the central server. IV. When the central server receives clients' improved models, it firstly uses taylor series expansion and fisher information matrix to compensate the model weights, then with the compensated weights, the central server calculates the nesterov momentum and updates the global model. V. The central server pushes the latest global model to the participating clients. For the architecture, the biggest difference between the Overlap-FedAvg and vanilla FedAvg is step III, where we relaxed the chain-like requirements, meaning the central server is not guaranteed for receiving the latest model from each client, which unavoidably brought staleness problem to our setting. Consequently, step IV managed to address the issue of stale model weights.

\begin{figure}[h]
	\centering
	\includegraphics[width=\linewidth]{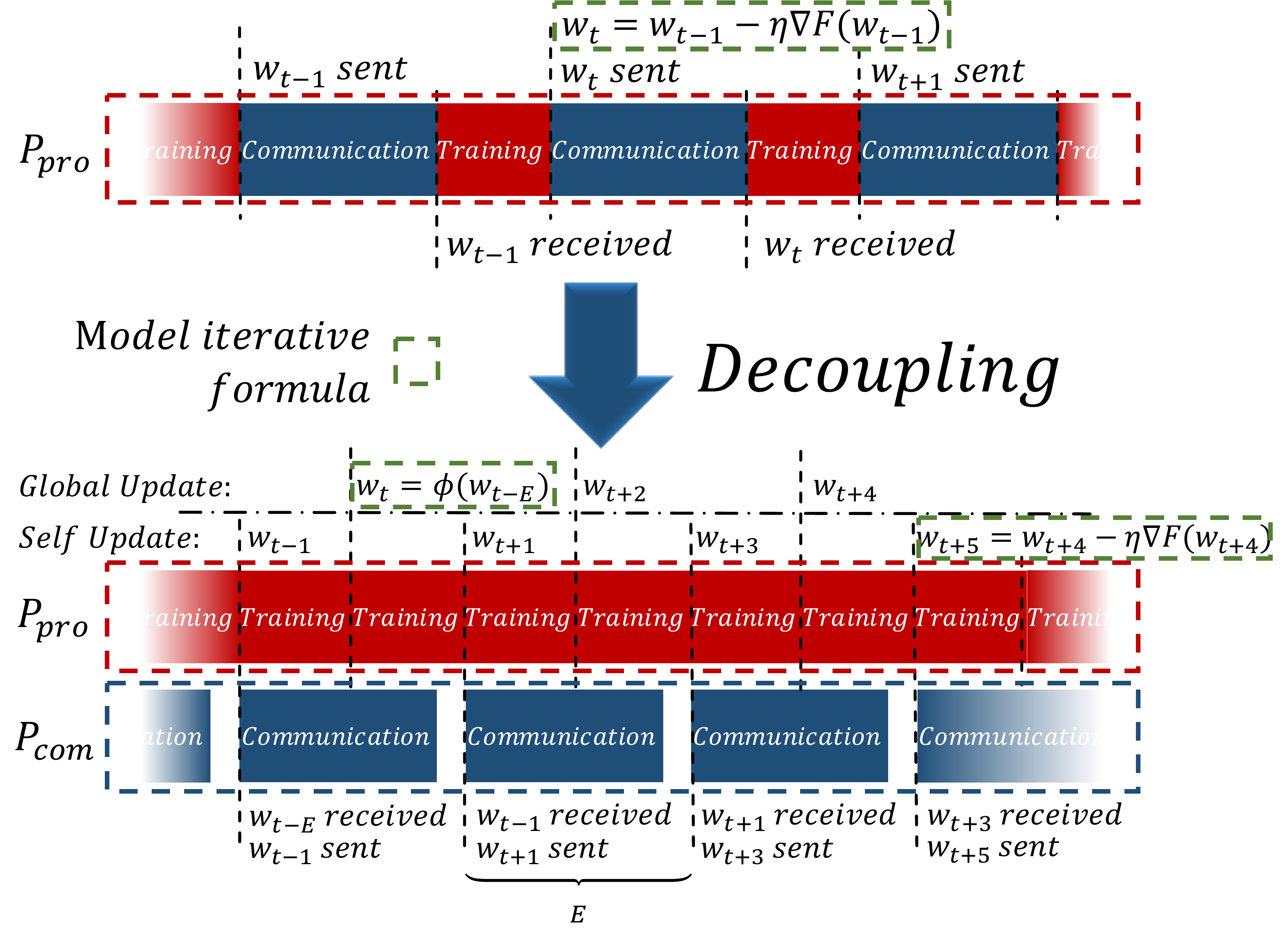}
	\caption{Elementary SGD was used in vanilla FedAvg for model iterating, which in the meantime restricted the training phase from paralleling with the communication phase, as the training phase required the model weights that the communication phase is transmitting. To decouple the training phase and communication phase, Overlap-FedAvg loosed this restriction and employed two model iterative formulas to the training process: \textit{self update} and \textit{global update}, where the former one was used when the required model weights for a global model updating are not received yet, and the latter one was used when the required model weights are successfully received.}
	\label{fig:arch_decoupling}
\end{figure}

\subsection{Overlapping of the Communication}
At the beginning of the federated learning in Overlap-FedAvg, each client will initialize two processes: a process focusing on local model training (denoted as $P_{pro}$) and a process dedicating to communicating (denoted as $P_{com}$). It is worth noting that the central server will also spawn two processes, where its $P_{pro}$ is used to do global model updating, and its $P_{com}$ is used to receive clients' improved models and pushes the latest global model to the clients. Since both feed-forward and back-propagation in model training is heavily depends on the solution of its previous step (i.e. $w_t$ is requisite for $w_{t+1} = w_{t} - \eta \nabla F(w_{t})$), we came to the conclusion that the model training process is strongly consecutive and real-time (i.e. chain-like). As a result, under vanilla FedAvg's setting, the model training at iteration $t+1$ will only begin after successfully receiving $w_t$, namely, after finishing the time-consuming communication phase at iteration $t$. Hence, in order to successfully decouple the model training phase from communicating, the real-time condition of vanilla FedAvg needs to be relaxed.

The diagram of Overlap-FedAvg's model iterating process is presented in Figure \ref{fig:arch_decoupling}. In each iteration, $P_{com}$ will fetch the model before $P_{pro}$ has not even started training (i.e. pre-download the global model), and it will also handle the model uploading task after the local model is improved, so that $P_{pro}$ can instantly continue the next iteration of local model training without caring the matters of communicating. From the diagram, we can see that the communication interval $E$ of Overlap-FedAvg is determined by the communication time and the training time. When $P_{com}$ communicates, $P_{pro}$, like vanilla FedAvg, utilizes elementary SGD and its local training data to update the local model. On the other hand, when $P_{com}$ finished communicating and received the latest model, unlike vanilla FedAvg that simply calculates the weighted aggregation as the new global model, Overlap-FedAvg use the function $\phi(\cdot)$ to update the global model, which will be detailedly discussed in the latter section. The pseudocode of Overlap-FedAvg is illustrated in Algorithm \ref{alg:overlap-fedavg}.

\begin{algorithm}[h] 
	\caption{\texttt{Overlap-FedAvg}.}
	\label{alg:overlap-fedavg} 
	\begin{algorithmic}[1] 
	    \PSDO
	        \State Initialization: global model $w_0$.
	        \For {each global iteration $t \in {1, ..., iteration}$}
	            \ForAllP {each client $k \in {1, ..., N}$ )}
	                \State \# Get clients improved model.
	                \State $TrainLocally(k, w_t)$
	            \EndForAllP
	            \If {$p_{com}~received~models$}
	                \State \# Update the global model.
	                \State $w_{t+1} = \phi(w_{t}, p_0, w_{t-E}^{0}, p_1, w_{t-E}^{1}, ..., p_N, w_{t-E}^{N})$
	            \EndIf
	        \EndFor
    \Blank
	\TrainLocally
	    \While {$p_{com}~is~communicating$}
	        \State \# Do self update
	        \State $w_{e} \leftarrow w_{e-1} - \eta \nabla F(w_{e-1})$
	    \EndWhile
	    \State \# Command $p_{com}$ to send latest model
	    \State $p\_com\_send(w_{E})$
	\end{algorithmic} 
\end{algorithm}

\subsection{Gradients Compensation}
\label{section:three:gradients-compensation}
The overlapping of the communication and training phase in the Overlap-FedAvg is promising because the communication is successfully hidden. However, the communication phase inevitably takes time, and from our observation, it is actually very likely to take much more time than a iteration of local model training. As a result, when the server/client finally received/uploaded the model, this model is probably from several iterations ago. Simply considering weighted aggregation of these stale models like FedAvg will mostly result in a degradation of model performance in Overlap-FedAvg.

\begin{figure}[t]
	\centering
	\includegraphics[width=\linewidth]{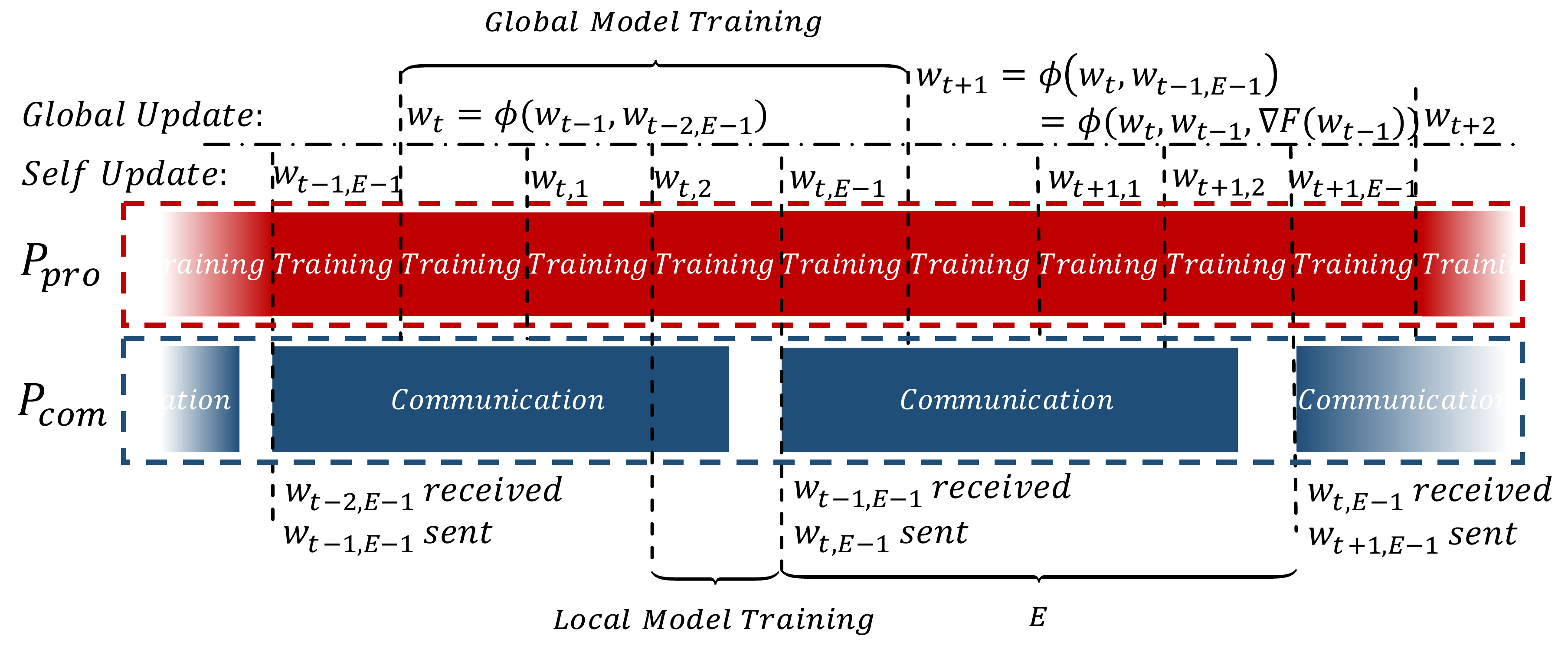}
	\caption{Compared to Figure~\ref{fig:arch_decoupling}, here the timestamp of $w$ will only increase after performing a global update. In other words, from the central server's perspective, $E$ iterations of clients' local model training are considered as one global model training in order to simplify the staleness analysis.}
	\label{fig:unityofE}
\end{figure}

To simplify the analysis of the degradation, we treat clients' $E$ iterations of local model training as one global model training, as Figure~\ref{fig:unityofE} illustrated, where $w_{t, i}^{k}$ is defined as the $k$-th client's model weights at the $i$-th local iteration in the $t$-th global iteration. Consequently, the staleness of the model is limited to $1$ global iteration (i.e. $E$ local iterations). Then, we represent the FedAvg's updating rule in Equality \ref{eq:fedavg-updating} for the latter comparison, since it does not produce stale models at all, and is what Overlap-FedAvg pursues for.

\begin{equation}\label{eq:fedavg-updating}
    w_{t+1} = \underbrace{\sum_{k=0}^{N} p_k w_{t+1}^{k} = w_{t} - \overbrace{\eta \sum_{k=0}^{N} p_k \sum_{e=0}^{E-1} \nabla F(w_{t,e}^{k})}^{w_{t+1}^k = w_{t} - \mu \sum_e \nabla F(w_{t, e}^{k})}}_{Global~Model~Updating},
\end{equation}

However, in Overlap-FedAvg, both $w_{t+1}^k$ and $\nabla F(w_{t,e}^{k})$ are not accessible when the central server updating the global model to $w_{t+1}$ due to the communication-computation overlap. Instead, only $w_{t}^k$ can be retrieved at that specific time. Thus, in Overlap-FedAvg, the updating rule is shown in Equality \ref{eq:ofa-global-update-1}.

\begin{equation}\label{eq:ofa-global-update-1}
w_{t+1} = \phi(w_t, p_0, w_{t}^{0}, p_1, w_{t}^{1}, ..., p_N, w_{t}^{N}).
\end{equation}

where $\phi(\cdot)$ is the core function to produce latest global model weights, and its goal is to generate outcomes that approach results from Equality~\ref{eq:fedavg-updating} at the same timestamp as much as possible. Considering $w_{t-1}$ can be readily obtained from history, and $w_t^k = w_{t-1} - \mu \sum_e \nabla F(w_{t-1,e}^k)$, Equality~\ref{eq:ofa-global-update-1} can be further decomposited into Equality~\ref{eq:ofa-global-update-2}

\begin{equation}\label{eq:ofa-global-update-2}
w_{t+1} = \phi(w_t, ..., p_k, w_{t}^{k}, w_{t-1}^{k}, ..., \sum_e \nabla F(w_{t-1,e}^{k}), ...).
\end{equation}

From the decomposition, we can see that in vanilla FedAvg, $w_{t+1}$ is updated by $w_{t}$ and each client's gradients at $w_{t}^{k}$ (i.e. $\nabla F(w_{t,e}^{k})$), which is expected. However, in Overlap-FedAvg, $w_{t+1}$ is improved by $w_{t}$ and the gradients at $w_{t-1}$, leading to the gradients staleness problem in Overlap-FedAvg. Hence, it is clear that our goal finally turns into estimating $\nabla F(w_{t}^{k})$ by $\nabla F(w_{t-1}^{k})$.

Inspired by the delay compensation in asynchronous distributed deep training~\cite{zheng2017asynchronous}, taylor series expansion and fisher information matrix were utilized to alleviate the gradients staleness problem. Specifically, by taylor series expansion, $\nabla F(w_{t}^{k})$ can be approximately represented by $\nabla F(w_{t-1}^{k})$:

\begin{equation}
\nabla F(w_{t}^{k}) \approx \nabla F(w_{t-1}^{k}) + \mathbb{H}(F(w_{t-1})) \odot (w_{t} - w_{t-1}),
\end{equation}

where $\mathbb{H}(\cdot)$ is the hessian matrix of $\cdot$. However, the hessian matrix is generally very difficult to compute and is against the lightweight characteristic of federated learning. Therefore, a cheap hessian matrix approximator was introduced with fisher information matrix\cite{friedman2001elements}:

\begin{equation}
\epsilon_t \triangleq \mathbb{E}_{y|x, w^*}\left|\left| G(w_t) - \mathbb{H}(F(w_t)) \right|\right| \rightarrow 0, t \rightarrow +\infty,
\end{equation}

where $G(w_t)$ is the outer product matrix of the gradients at $w_t$.

Additionally, in order to control the variance of the approximator, $\lambda$ was imported. Thus, $\nabla F(w_{t}^{k})$ can be cheaply approximated by $\nabla F(w_{t-1}^{k})$ as:

\begin{equation}
\nabla F(w_{t}^{k}) \approx \nabla F(w_{t-1}^{k}) + \lambda \nabla F(w_{t-1}) \odot \nabla F(w_{t-1}) \odot (w_{t} - w_{t-1}).
\end{equation}

Consequently, we had an asymptotically unbiased and cheap approximator to represent the $F(w_{t}^{k})$ from $F(w_{t-1}^{k})$.

\subsection{Application of NAG}

Compared to local model training with pre-collected data or vanilla distributed deep learning, federated learning obtained good privacy, but lost the relatively fast convergence speed. As a result, some work was recently investigating the possibility of utilizing momentum in federated learning\cite{liu2020accelerating, huo2020faster} to accelerate its convergence speed. However, this work was focusing on representing the momentum with model weights $w_t$, which is slightly against the instinct as momentum is essentially describing the trend of gradients. Hence, their algorithms are relatively difficult to understand compared to normal momentum-based acceleration methods (e.g. NAG, Adam\cite{kingma2014adam}) as momentum is superficially unrelated to the model weights.

In Overlap-FedAvg, with the decomposition of the original model updating rule and the abstraction and estimation of the latest gradients from the stale model weights, many parameters optimizing acceleration techniques can be applied to the Overlap-FedAvg intuitively and easily. Particularly, the NAG, one of the most popular approaches to speed-up the training process of the DNN, was utilized to accelerate the Overlap-FedAvg in this proposal as a example. The details and analysis are provided.

In detail, the normal NAG updating rule in SGD is illustrated as follows:

\begin{equation}
\begin{cases}
	v_{t+1} &= \beta v_{t} + \nabla F(w_{t}) + \beta(\nabla F(w_{t}) - \nabla F(w_{t-1})) \\
	w_{t+1} &= w_{t} - \eta v_{t+1}.
\end{cases}
\end{equation}

Since we already had the unbiased estimation of $\nabla F (w_{t})$ with $\nabla F (w_{t-1})$, the NAG updating rule in Overlap-FedAvg could be written as:

\begin{equation}
\begin{cases}
	v_{t+1} &= \beta v_{t} + \nabla F^{ah}(w_{t-1}) \\
	&\qquad + \beta(\lambda \nabla F(w_{t-1}) \odot \nabla F(w_{t-1}) \odot (w_{t} - w_{t-1}))\\
	w_{t+1} &= w_{t} - \eta v_{t+1},
\end{cases}
\end{equation}

where $\nabla F^{ah}(\cdot)$ is denoted as the compensated gradients with ($\cdot$).

\subsection{The overall model updating formula}

Overall, when the central server updated the global model in the Overlap-FedAvg, the following procedures were executed:

\begin{enumerate}
    \item Restoring the gradients of received model weights by $\nabla F(w_{t-1}) = \frac{w_{t-1} - w_{t}}{\eta}$. This step treated several iterations (say $E$ iterations) of clients' local training as one global model training with the same learning rate but $E$ times larger gradient values. This behavior was expected because we want Overlap-FedAvg to have an identical learning speed regardless of how the adaptive $E$ variants, and it matches the chain-like characteristic of SGD. Although the gradient values are scaled up (if $E$ is not equal to $1$), the scaling factor is a fixed constant (namely $E$) and can be extracted with ease, and we argue that it is not a big problem to the latter transformation.
    \item Compensating $\nabla F(w_{t-1})$ to $\nabla F^{ah}(w_{t-1})$ with the formula proposed in Section \ref{section:three:gradients-compensation}.
    \item With the unbiased estimation of the gradients on the current model weights, the NAG algorithm was employed to speed up the federated learning process.
\end{enumerate}

Specifically, we had the updating function $\phi(\cdot)$ shown in Algorithm \ref{alg:overlap-fedavg update}:

\begin{algorithm}[h] 
	\caption{\texttt{$\phi(w_{t}, p_0, w_{t-E}^{0}, p_1, w_{t-E}^{1}, ..., p_N, w_{t-E}^{N})$}.}
	\label{alg:overlap-fedavg update} 
	\begin{algorithmic}[1] 
	    \State \# Restore the gradients of the last global iteration.
	    \For {$k \in 1, ..., N$}
	        \State $\nabla F(w_{t-1}^k) = \frac{w_{t-1}^k - w_{t}}{\eta}$
	    \EndFor
	    \Blank
	    \State \# Calculate the weighted gradients for the central server.
	    \State $\nabla F(w_{t-1}) = \sum_{k=0}^{N} p_k \nabla F(w_{t-1}^k)$
	    \Blank
	    \State \# Compensate the stale gradients
	    \State $\nabla F^{ah}(w_{t-1}) = \nabla F(w_{t-1}) + \lambda \nabla F(w_{t-1}) \odot \nabla F(w_{t-1}) \odot (w_{t} - w_{t-1})$
	    \Blank
	    \State \# Update the momentum and update the globla model
	    \State $v_{t+1} = \beta v_{t} + \nabla F^{ah}(w_{t-1}) + (\nabla F^{ah}(w_{t-1}) - \nabla F(w_{t-1}))$
	    \State $w_{t+1} = w_{t} - \eta v_{t+1}$
	\end{algorithmic} 
\end{algorithm}

\section{Theoretical Analysis}
\label{section:four}
In this section, we provide the theoretical analysis of the proposed Overlap-FedAvg framework. Specifically, the goal of the Overlap-FedAvg is to solve the following generic optimization problem, which is formulated as Equation~\ref{eq:difination1}:

\begin{equation}\label{eq:difination1}
   \min_{w} \left\{ F(w) \triangleq \sum_{k=1}^{N} p_{k}F_{k}(w) \right\},
\end{equation}

where $N$ is the number of devices, $p_{k}$ is the aggregation weight of the $k$-th client satifying $p_{k} \geq 0$ and $\sum_{k=1}^{N}p_{k}=1$, $F_{k}( \cdot )$ refers to clients' local objective function defined by Equation~\ref{eq:difination2}:

\begin{equation}\label{eq:difination2}
F_{k}(w) \triangleq \frac{1}{n_{k}} \sum_{i=1}^{n_{k}} \mathcal{L}(w; x_{k, i}),
\end{equation}

where $x_{k, i}$ refers to the $k$-th client's $i$-th local training data, $n_{k}$ is the size of the $k$-th client's training data and $\mathcal{L}( \cdot )$ is the loss function.

In Overlap-FedAvg, each client will perform $E$ iterations of local training where $E$ is self-adaptive to the environment instead of manually configured in vanilla FedAvg, and then perform a local improved model uploading procedure to the central server for it to update the global model. Therefore, from the perspective of the central server, considering $E$ iterations of clients' local model training as one iteration of central server's global model training (which is explained in 
Section~\ref{section:three:gradients-compensation}), the global model's update rule can be illustrated in Equation~\ref{eq:global_model_update}:

\begin{align}\label{eq:global_model_update}
\begin{split}
\nabla F_{k}^{ah}(w_{t-1}) &= \nabla F_{k}(w_{t-1}) + 
\\ &\qquad\lambda \nabla F_{k}(w_{t-1}) \odot \nabla F_{k}(w_{t-1}) \odot (w_{t} - w_{t-1})
\end{split} \\
w_{t+1} &= w_{t} - \eta_{t} \sum_{k=1}^{N} p_{k} \nabla F_{k}^{ah}(w_{t-1}).
\end{align}

In order to present our theorem, some mild assumptions need to be made.

\begin{assumption}
\label{assump:1}
	For the $k$-th client, where $k \in range(1, N)$, $F_k$ is $L_1$-smooth. That is, for all $x$ and $y$:
	\begin{equation}
	    F_{k}(y) \leq F_{k}(x) + (y - x)^{T} \nabla F_{k}(x) + \frac{L_1}{2}\left|\left| y-x \right|\right|_{2}^{2}.
	\end{equation}
\end{assumption}

\begin{assumption}
\label{assump:2}
	\cite{lian2015asynchronous}\cite{lee2016gradient} For the $k$-th client, where $k \in range(1, N)$, $\nabla F_k$ is $L_2$-smooth. That is, for all $x$ and $y$:
	\begin{equation}
	    \nabla F_{k}(y) \leq \nabla F_{k}(x) + (y - x)^{T} \mathbb{H}(F_{k}(x)) + \frac{L_1}{2}\left|\left| y-x \right|\right|_{2}^{2}.
	\end{equation}
	where $\mathbb{H}(\cdot)$ is the second order derivative.
\end{assumption}

\begin{assumption}
\label{assump:3}
	For the $k$-th client, where $k \in range(1, N)$, $F_k$ is $\mu$-strongly convex. That is, for all $x$ and $y$:
	\begin{equation}
	    F_{k}(y) \geq F_{k}(x) + (y - x)^{T} \nabla F_{k}(x) + \frac{\mu}{2}\left|\left| y-x \right|\right|_{2}^{2}.
	\end{equation}
\end{assumption}

A recent work\cite{zheng2017asynchronous} had demonstrated the convergence rate of delay compensated asynchronous SGD under distributed deep learning's setting, and Theorem \ref{theorem:dc} is derived from its main theorem.

\begin{theorem} \label{theorem:dc}
	Let Assumption \ref{assump:1} and Assumption \ref{assump:2} hold, if the expectation of the $\left|\left| \cdot \right|\right|_{2}^{2}$ norm of the $\nabla F_{k}^{ah}$ is upper bounded by a constant G, and the diagonalization error of $\mathbb{H}(\cdot)$ is upper bounded by $\epsilon_D$, then the difference between the real gradients and approximated gradients at $w_t$ is
	\begin{equation}
	    \left|\left| \nabla F_{k}(w_{t}) - \nabla F_{k}^{ah}(w_{t-1}) \right|\right| \leq G \eta_t (\frac{L_1 \eta_t}{2} + C_\lambda + \epsilon_t),
	\end{equation}
	where $C_\lambda=(1-\lambda)L_{1}^2 + \epsilon_D$ , $\epsilon_t \triangleq \mathbb{E}_{(y|x,w^*)}\left|\left| \nabla F_{k}(w_t) - \mathbb{H}(F_{k}(w_t)) \right|\right|$ \cite{friedman2001elements} and $\eta_t$ is the learning rate at iteration $t$.
\end{theorem}

\begin{proof}
\begingroup
\allowdisplaybreaks
\begin{align}\label{eq:dc_theorem}
&\qquad \left|\left|\nabla F_{k}(w_{t}) - \nabla F_{k}^{ah}(w_{t-1}) \right|\right| \nonumber \\ 
&= \underbrace{
\left|\left| \nabla F_{k}(w_t) - \nabla F_{k}^{h}(w_{t-1}) \right|\right|}_{A} + \underbrace{
\left|\left| \nabla F_{k}^{h}(w_{t-1}) - \nabla F_{k}^{ah}(w_{t-1}) \right|\right|}_{B},
\end{align}
\endgroup

where $\nabla F_{k}^{h}(\cdot)$ is denoted as the first order approximation of $\nabla F_{k}(\cdot)$.

For the term $A$, as $\nabla F_k$ is $L_2$-smooth, we have Inequality~\ref{eq:dc_1term} as:
\begingroup
\allowdisplaybreaks
\begin{align}\label{eq:dc_1term}
	A &= \left|\left| \nabla F_{k}(w_t) - \nabla F_{k}^{h}(w_{t-1}) \right|\right| \nonumber \\ 
	&\leq \frac{L_2}{2}\left| \left| w_{t} - w_{t-1} \right|\right|^2 \leq \frac{L_1 G}{2}\eta_{t}^{2}.
\end{align}
\endgroup

For the term $B$, we have Inequality~\ref{eq:dc_2term} as:
\begingroup
\allowdisplaybreaks[1]
\begin{align}
	&B = \left|\left| \nabla_{k}^{h}(w_{t-1}) - F_{k}^{ah}(w_{t-1}) \right|\right| \nonumber \\
	&\leq \left|\left| (\lambda \nabla F_k(w_{t-1}) \odot \nabla F_k(w_{t-1}) - \mathbb{H}(F(w_{t-1})))(w_{t} - w_{t-1}) \right|\right| \nonumber \\
	&\leq (\left|\left| \lambda \nabla F_k(w_{t-1}) \odot \nabla F_k(w_{t-1}) - \nabla F_k(w_{t-1}) \odot \nabla F_k(w_{t-1}) \right|\right| \nonumber \\
	&\qquad + \left|\left| \nabla F_k(w_{t-1}) \odot \nabla F_k(w_{t-1}) - Diag(\mathbb{H}(F(w_{t-1}))) \right|\right| \nonumber \\
	&\qquad + \left|\left| Diag(\mathbb{H}(F(w_{t-1}))) - \mathbb{H}(F(w_{t-1}))\right|\right|)G \eta_t \nonumber \\
	&\leq (C_\lambda + \epsilon_t) G \eta_t \label{eq:dc_2term},
\end{align}
\endgroup

where $C_\lambda=(1-\lambda)L_{1}^2 + \epsilon_D$. Taking Inequality \ref{eq:dc_1term} and Inequality \ref{eq:dc_2term} into Equation \ref{eq:dc_theorem}, we have Inequality~\ref{eq:dc_bound}:
\begingroup
\allowdisplaybreaks
\begin{align}\label{eq:dc_bound}
	\left|\left|\nabla F_{k}(w_{t}) - \nabla F_{k}^{ah}(w_{t-1}) \right|\right| \leq G \eta_t (\frac{L_1 \eta_t}{2} + C_\lambda + \epsilon_t),
\end{align}
\endgroup
where $C_\lambda=(1-\lambda)L_{1}^2 + \epsilon_D$.
\end{proof}

Recently another work\cite{li2019convergence} shared a theoretical analysis of the vanilla FedAvg algorithm on the Non-IID and IID dataset and gave a necessary condition for the vanilla FedAvg on Non-IID dataset to converge: $\eta$ must decay. As Theorem \ref{theorem:dc} shown, with $\eta \rightarrow 0$, $\left|\left|\nabla F_{k}(w_{t}) - \nabla F_{k}^{ah}(w_{t-1}) \right|\right| \rightarrow 0$ as well. Hence, intuitively Overlap-FedAvg converges as long as FedAvg converge. We will prove this intuition formally.

Another theoretical analysis of federated learning~\cite{li2019convergence} defined $v_{t+1}^{k} = w_t^{k} - \eta_t \nabla F_{k}(w_{t}^{k}, \xi_{t}^{k})$, $g_t=\sum_{k=1}^{N} p_k \nabla F_k(w_t^k, \xi_{t}^{k})$, $\overline{v}_{t} = \sum_{k=1}^{N}p_k v_{t}^{k}$, $\overline{w}_t = \sum_{k=1}^{N}p_k w_{t}^{k}$ and $\overline{g}_{t} = \sum_{k=1}^{N} p_k \nabla F_{k}(w_{t}^{k})$ to help abstract the problem, where $\xi_{t}^{k}$ is a sample uniformly chosen from the $k$-th client's local training dataset. We will inherit these notations to our analysis. Furthermore, we denote the compensated $g_t$ as $g_t^{ah}$ and compensated $\overline{g}_t$ as $\overline{g}_t^{ah}$. Thus, $\left|\left| g_t^{ah} - g_t \right|\right| \leq \epsilon_c$

The main theorem of the literature~\cite{li2019convergence} is heavily based on its Key Lemma $1-3$. Among them the Key Lemma $1$ is the most important, and it is also the only lemma that will be interfered by the introduction of $\epsilon_c$. For the simplicity, we will denote $G \eta_t (\frac{L_1 \eta_t}{2} + C_\lambda + \epsilon_t)$ as $\epsilon_{c}$. Consequently, we will only give a revised version of its Key Lemma $1$ under Overlap-FedAvg settings.

\begin{theorem}
\label{theorem:lemma1}
	Let Assumption \ref{assump:1} to \ref{assump:3} hold. If $\eta_{t} \leq \frac{1}{4L_{1}}$, we have:
	\begingroup
	\allowdisplaybreaks
	\begin{align}
		&\mathbb{E}\left|\left| \overline{v}_{t+1} - w^{*} \right|\right|^2 \leq (1 - \eta_{t}\mu) \mathbb{E} \left|\left| \overline{w}_{t} - w^{*} \right|\right|^2 \nonumber \\
		&\qquad + \eta_{t}^{2} \mathbb{E} \left|\left| g_{t} - \overline{g}_{t} \right|\right|^2 + 6L_{1}\eta_{t}^{2} \Gamma + 2 \mathbb{E} \sum_{k=1}^{N} p_{k} \left|\left| \overline{w}_{t} - w_{t}^{k} \right|\right|^2 \nonumber \\
		&\qquad + 2\eta_{t}^{2}\epsilon_c G + \eta_t^{2} \epsilon_{c}^{2} + 2 \eta_{t} \epsilon_c \left|\left| \overline{w}_{t} - w^{*} \right|\right|,
	\end{align}
	\endgroup
	where $\Gamma = F^{*} - \sum_{k=1}^{N} p_{k} F_{k}^{*} \geq 0$ and $\epsilon_{c} = G \eta_t (\frac{L_1 \eta_t}{2} + C_\lambda + \epsilon_t)$
\end{theorem}

\begin{proof}
Because $\overline{v}_{t+1} = \overline{w}_t - \eta_t g_{t}^{ah}$, we have Equation~\ref{eq:conv-define} as:

\begingroup
\allowdisplaybreaks
\begin{align}\label{eq:conv-define}
	&\left|\left| \overline{v}_{t+1} - w^* \right|\right| = \left|\left| \overline{w}_t - \eta_{t}g_{t}^{ah} - w^* - \eta_{t}\overline{g}_{t}^{ah} + \eta_{t} \overline{g}_t^{ah} \right|\right|^2 \nonumber \\
	&= \underbrace{\left|\left| \overline{w}_t - w^* - \eta_{t}\overline{g}_{t}^{ah} \right|\right|^2}_{C} + \underbrace{2\eta_{t} \left< \overline{w}_t - w^* - \eta_{t}\overline{g}_{t}^{ah}, \overline{g}_{t}^{ah} - g_{t}^{ah} \right>}_{D} \nonumber \\
	&\qquad + \eta_{t}^2 \left|\left| g_{t}^{ah} - \overline{g}_{t}^{ah} \right|\right|,
\end{align}
\endgroup

where $w^*$ is the model weights we are pursing. For the second term $D$, because of $\mathbb{E}g_t^{ah} = \overline{g}_{t}^{ah}$, the expectation of $D$ is equal to $0$. For the first term $C$, after splitting we will have Equation~\ref{eq:lemma1}:

\begingroup
\allowdisplaybreaks
\begin{align} \label{eq:lemma1}
	&C = \left|\left| \overline{w}_t - w^* - \eta_{t}\overline{g}_{t}^{ah} \right|\right|^2 \nonumber \\
	&= \left|\left| \overline{w}_t - w^* \right|\right|^2 \underbrace{- 2\eta_{t} \left< \overline{w}_t - w^*, \overline{g}_t^{ah} \right>}_{E} + \underbrace{\eta_t^2 \left|\left| \overline{g}_{t}^{ah} \right|\right|^2}_{F}.
\end{align}
\endgroup

From Assumption \ref{assump:1}, it can be derived that $\left|\left| \nabla F_{k}(w_{t}^{k}) \right|\right| \leq 2L_1(F_{k}{w_{t}^{k}} - F_{k}^{*})$. As a result, for the term $F$, we have Inequality \ref{eq:F-1}:
\begingroup
\allowdisplaybreaks
\begin{align}\label{eq:F-1}
	F &= \eta_t^2 \left|\left| \overline{g}_{t}^{ah} \right|\right|^2 \leq \eta_{t}^{2} \sum_{k=1}^{N} p_{k} \left|\left| \nabla F_{k}^{ah}(w_{t}^{k}) \right|\right|^2 \nonumber \\
	&\leq \eta_{t}^{2} \sum_{k=1}^{N} p_{k} \left|\left| \nabla F_{k}(w_{t}^{k}) + \epsilon_c \right|\right|^2 \nonumber \\
	&\leq \eta_{t}^{2} \sum_{k=1}^{N} p_{k} (\left|\left| \nabla F_{k}(w_{t}^{k})\right|\right|^2 + 2\left< \nabla F_{k}(w_{t}^{k}), \epsilon_c \right> + \left|\left| \epsilon_c \right|\right|^2) \nonumber \\
	&\leq 2L_1 \eta_{t}^{2}\sum_{k=1}^{N} p_{k}(F_k(w_{t}^{k}) - F_{k}^{*}) + 2\eta_{t}^{2}\epsilon_c \sum_{k=1}^{N} p_{k} \nabla F_{k}(w_{t}^{k}) + \eta_t^{2} \epsilon_{c}^{2} \nonumber \\
	&\leq 2L_1 \eta_{t}^{2}\sum_{k=1}^{N} p_{k}(F_k(w_{t}^{k}) - F_{k}^{*}) + 2\eta_{t}^{2}\epsilon_c G + \eta_t^{2} \epsilon_{c}^{2}.
\end{align}
\endgroup

And thus $F$ is successfully bounded. Next we target at bounding the second term $E$, which is also not too difficult as we only need to bring the disturbance $\epsilon_c$ to the equality. Formally, for the second term $E$, we have Inequality~\ref{eq:E-1}:
\begingroup
\allowdisplaybreaks
\begin{align}\label{eq:E-1}
	E &= -2\eta_{t} \left< \overline{w}_t - w^*, \overline{g}_t^{ah} \right> = -2 \eta_{t} \sum_{k=1}^{N} p_{k} \left< \overline{w}_t - w^{*}, \nabla F_{k} (w_{t}^{k}) \right> \nonumber \\
	&\leq -2 \eta_{t} \sum_{k=1}^{N} p_{k} \left< \overline{w}_{t} - w_{t}^{k}, \nabla F_{k}(w_{t}^{k}) \right> + 2 \eta_{t} \sum_{k=1}^{N} p_{k} \left< w_{t}^{k} - w^{*}, \epsilon_c \right> \nonumber \\
	&\qquad + 2 \eta_{t} \sum_{k=1}^{N} p_{k} \left< w_{t}^{k} - w^{*}, \nabla F_{k}(w_{t}^{k}) \right> \nonumber \\
	&\leq -2 \eta_{t} \sum_{k=1}^{N} p_{k} \left< \overline{w}_{t} - w_{t}^{k}, \nabla F_{k}(w_{t}^{k}) \right> + 2 \eta_{t} \epsilon_c \left|\left| \overline{w}_{t} - w^{*} \right|\right| \nonumber \\
	&\qquad - 2 \eta_{t} \sum_{k=1}^{N} p_{k} \left< w_{t}^{k} - w^{*}, \nabla F_{k}(w_{t}^{k}) \right>.
\end{align}
\endgroup

Here we want to eliminate all $\nabla F_k(\cdot)$. By Cauchy-Schwarz inequality, AM-GM inequality and Assumption \ref{assump:3}, the inequality can be formalized as Inequality~\ref{eq:E-2}:

\begin{align}\label{eq:E-2}
\begin{split}
	E &= -2\eta_{t} \left< \overline{w}_t - w^*, \overline{g}_t^{ah} \right> \\ 
	& \leq \eta_{t} \sum_{k=1}^{N} p_{k} \left( \frac{1}{\eta_{t}} \left|\left| \overline{w}_{t} - w_{t}^{k} \right|\right|^2 + \eta_{t} \left|\left| \nabla F_{k}(w_{t}^{k}) \right|\right|^2 \right) \\
	&\qquad - 2 \eta_{t} \sum_{k=1}^{N} p_{k} \left( F_{k}(w_{t}^{k}) - F_{k}(w^{*}) + \frac{\mu}{2} \left|\left| w_{t}^{k} - w^{*} \right|\right|^2 \right) \\
	&\qquad + 2 \eta_{t} \epsilon_c \left|\left| \overline{w}_{t} - w^{*} \right|\right|.
\end{split}
\end{align}

Putting these terms back to $C$ and following the induction in \cite{li2019convergence}, we can derive the revised version of the Key Lemma $1$, which is Theorem \ref{theorem:lemma1} in this paper.

\end{proof}

\begin{table}[tb]
  \centering
  \captionsetup{justification=centering}
  \caption{\\\textsc{The basic configuration of the transformer model}}
    \begin{tabular}{cccccc}
    \toprule
    \toprule
    \multirow{2}[1]{*}{Params}&\# of word&\# of& \# of Hidden &\# of & Dropout \\
          & Embeddings & Head & Units & Layers & Rate\\
    \midrule
    Value & 200   & 2     & 200   & 2     & 0.2  \\
    \bottomrule
    \bottomrule
    \end{tabular}%
  \label{tab:transformer-arch}%
\end{table}%

\begin{figure}[htb]
    \centering
    \subfloat[]{%
        \includegraphics[width=\linewidth]{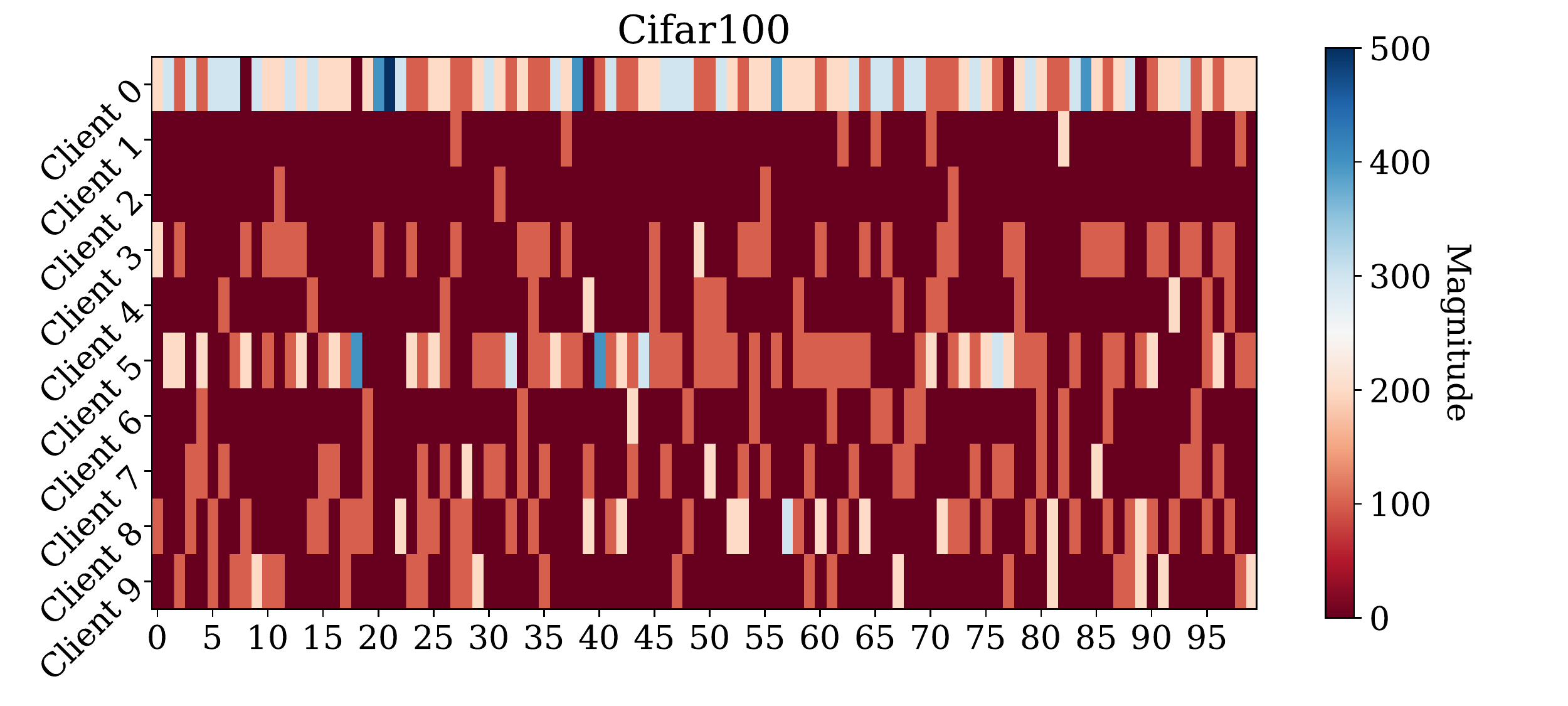}}
    \\
    \subfloat[]{%
       \includegraphics[width=0.5\linewidth]{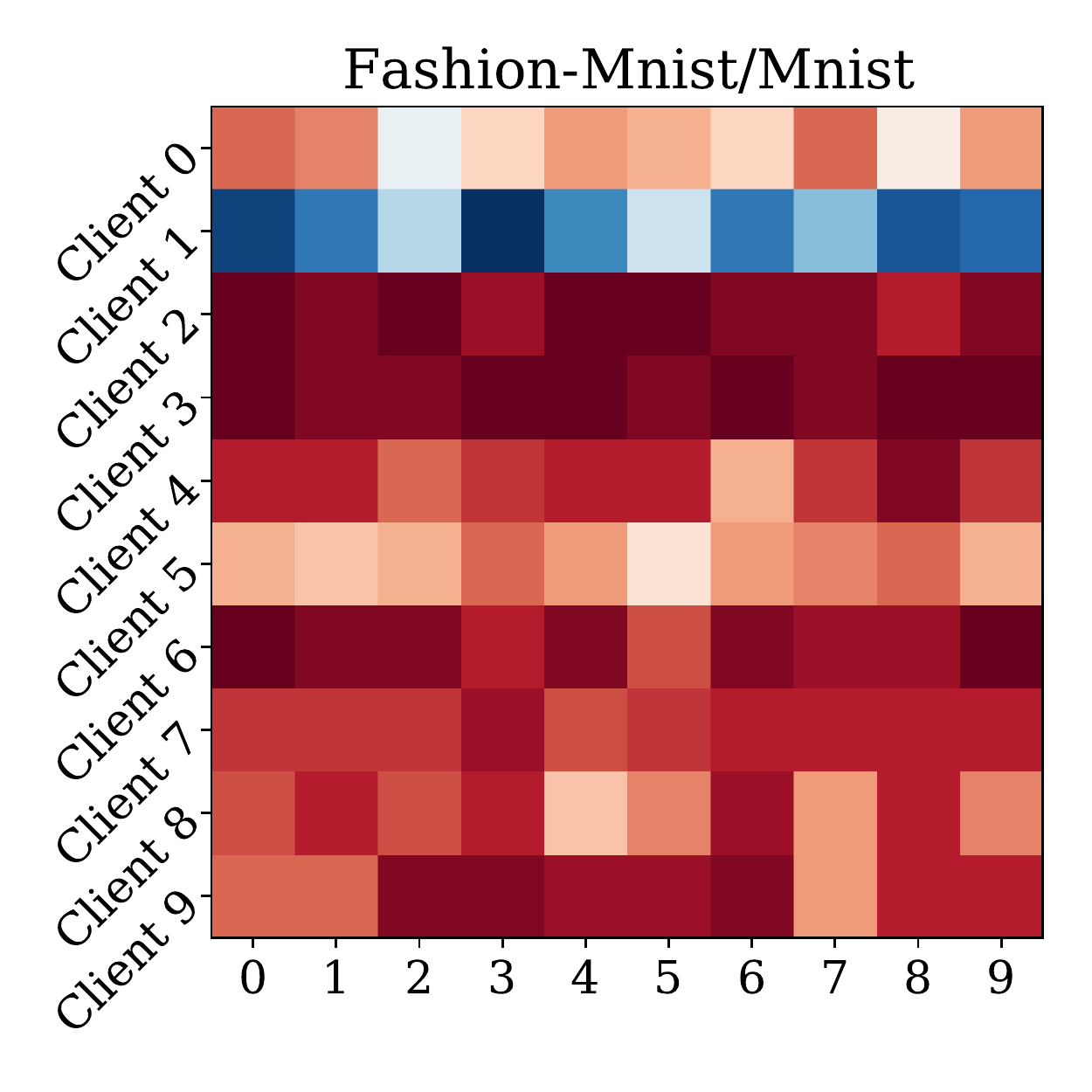}}
    \hfill
    \subfloat[]{%
        \includegraphics[width=0.5\linewidth]{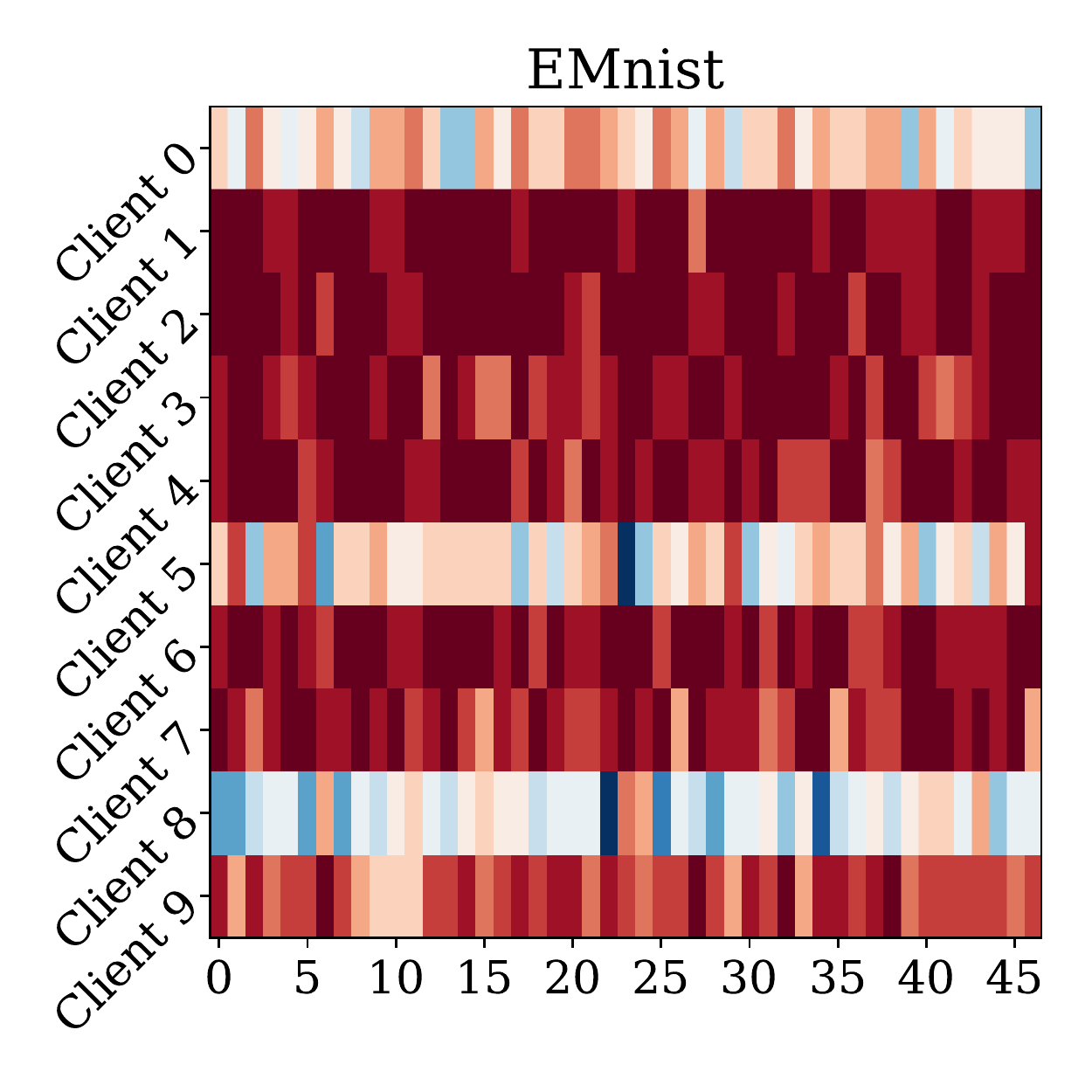}}
    \\
    \subfloat[]{%
        \includegraphics[width=0.5\linewidth]{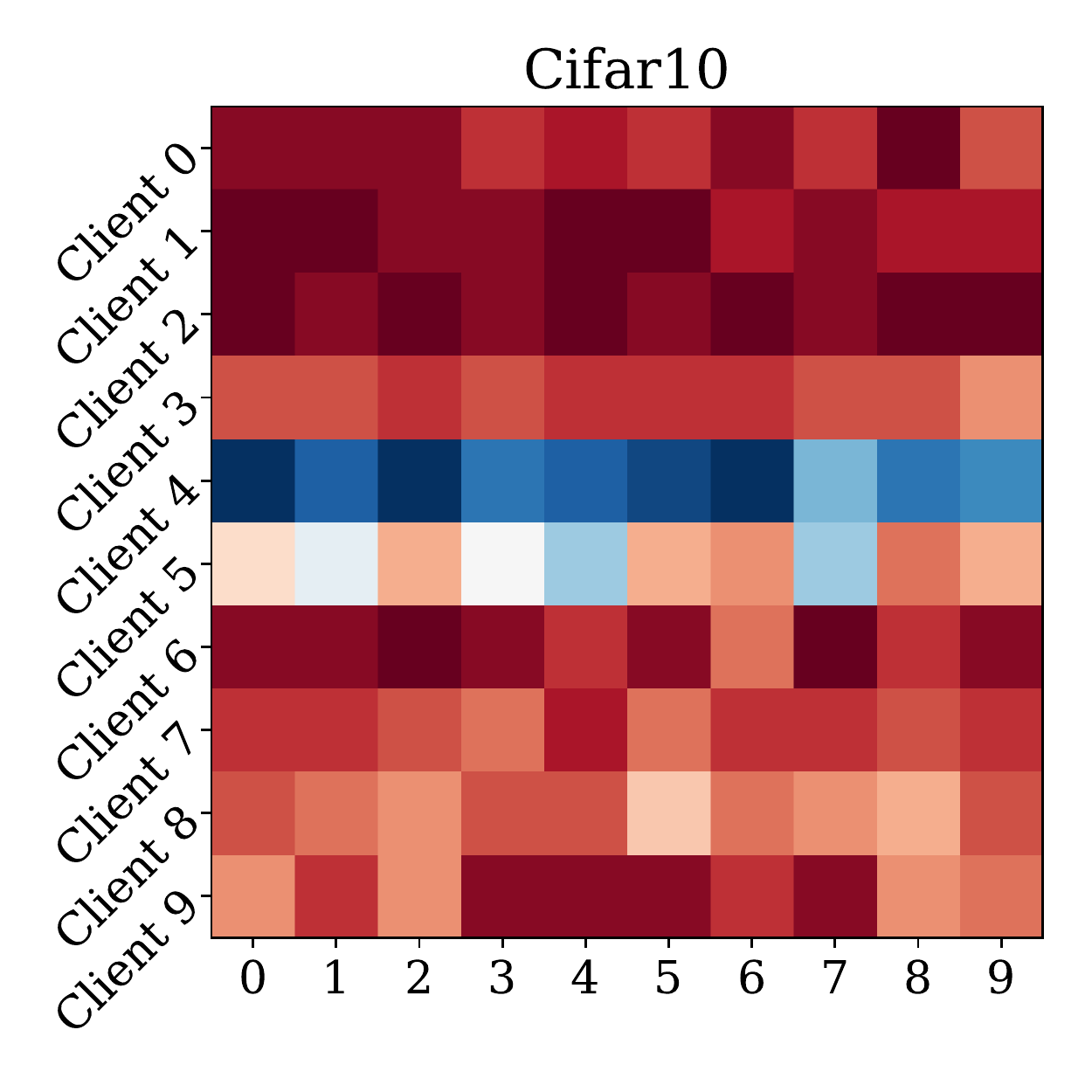}}
    \hfill
    \subfloat[]{%
        \includegraphics[width=0.5\linewidth]{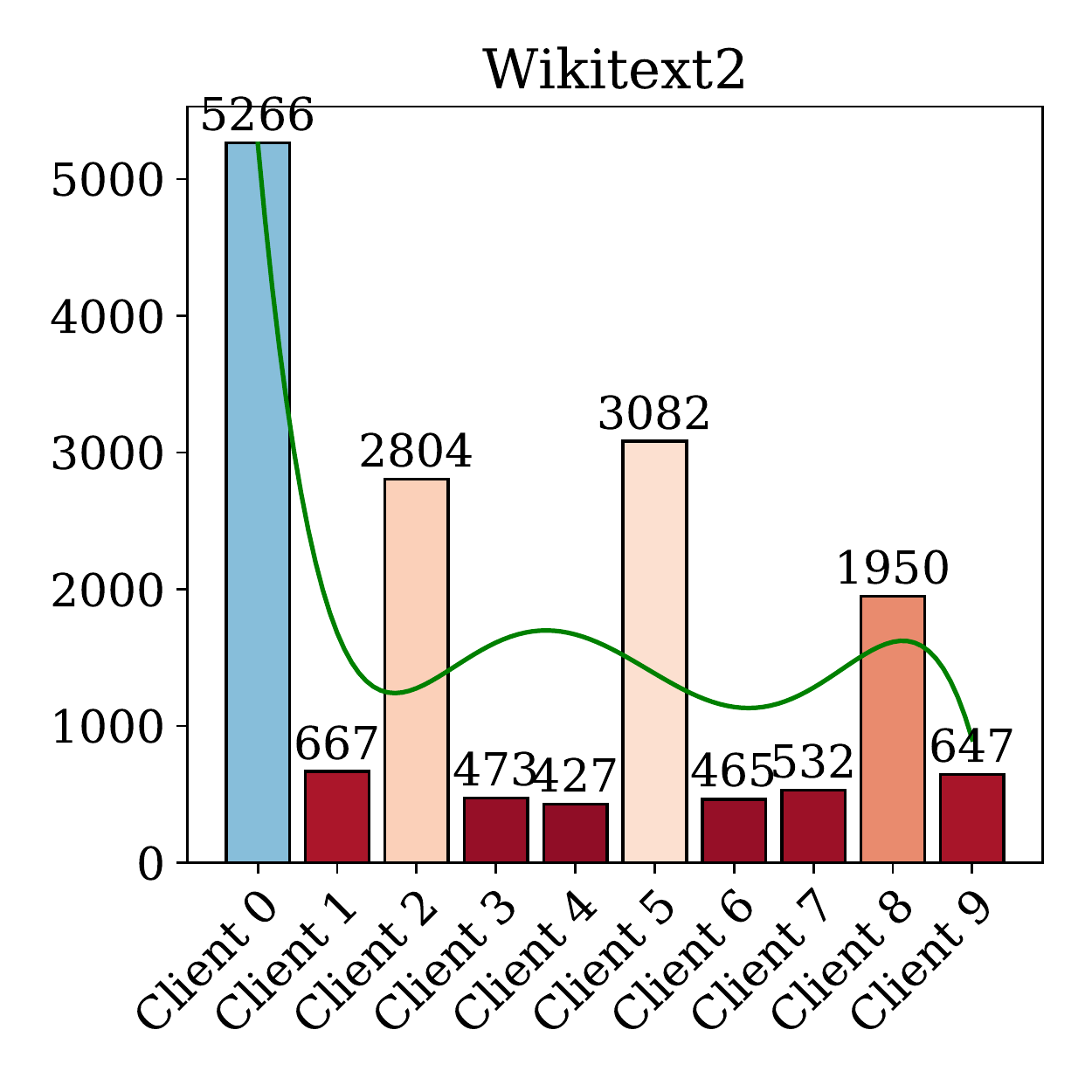}}
    \caption{Each sub-figure has two axes: a client axis and a label (class) axis, and the color of each cell or bar represents the magnitude of the assigned data. (a) The Non-IID partition of the Cifar100 dataset, which has in total 100 classes and 50000 training samples. (b) The Non-IID partition of the Fashion Mnist dataset and the Mnist dataset, which have in total 10 classes and 60000 training samples. (c) The Non-IID partition of the EMnist dataset, which has in total 47 classes and 112800 training samples. (d) The Non-IID partition of the Cifar10 dataset, which has in total 10 classes and 50000 training samples. (e) The Non-IID partition of the Wikitext2 dataset, which has over 16000 training samples.}
    \label{fig:dataset-distribution} 
\end{figure}

\begin{table*}[!b]
  \centering
  \captionsetup{justification=centering}
  \caption{\\\textsc{Accuracy comparison between vanilla FedAvg and Overlap-FedAvg with different $\lambda$ and fixed $\beta = 0.2$}}
    \begin{tabular}{cccccccc}
    \toprule
    \toprule
    \multirow{3}[4]{*}{Model} & \multirow{3}[4]{*}{dataset} & \multirow{3}[4]{*}{$\eta$} & FedAvg & \multicolumn{4}{c}{Overlap-FedAvg} \\
\cmidrule(lr){4-4} \cmidrule(lr){5-8}     &       &       & \multirow{2}[2]{*}{Accuracy / PPL} & $\lambda=0.0$ & $\lambda=0.2$ & $\lambda=0.5$ & $\lambda=0.8$ \\
          &       &       &       & Accuracy / PPL & Accuracy / PPL & Accuracy / PPL & Accuracy / PPL \\
    \midrule
    MLP   & Mnist & 0.001 & 0.9711 & 0.9776 & 0.9775 & 0.9775 & 0.9777 \\
    MnistNet & Fmnist & 0.001 & 0.9130 & 0.9146 & 0.9145 & 0.9153 & 0.9143 \\
    MnistNet & EMNist & 0.001 & 0.8661 & 0.8672 & 0.8676 & 0.8673 & 0.8676 \\
    CNNCifar & Cifar10 & 0.001 & 0.5088 & 0.4970 & 0.5060 & 0.5058 & 0.5058 \\
    VGG$^R$ & Cifar10 & 0.0001 & 0.4321 & 0.4272 & 0.4247 & 0.4248 & 0.4247 \\
    ResNet$^R$ & Cifar10 & 0.0001 & 0.4356 & 0.4341 & 0.4348 & 0.4345 & 0.4345 \\
    ResNet$^R$ & Cifar100 & 0.0001 & 0.0895 & 0.0865 & 0.0866 & 0.0866 & 0.0866 \\
    Transformer & Wikitext-2 & 0.0001 & 547.067 & 546.966 & 546.921 & 546.920 & 546.920 \\
    \midrule
    \multicolumn{8}{c}{$\beta=0.0$} \\
    \bottomrule
    \bottomrule
    \end{tabular}%
  \label{tab:accuracy-lambda}%
\end{table*}%

Comparing Theorem \ref{theorem:lemma1} with the Key Lemma $1$ originated in literature~\cite{li2019convergence}, several terms that Overlap-FedAvg introduced are strongly relevant to the learning rate $\eta_{t}$. Therefore, with the decaying of the learning rate during the federated training process on the IID or Non-IID dataset, these introduced terms will gradually approach nearly zero and eventually become negligible for the convergence of the learning.
    
\section{Experiments}
\label{section:five}

\subsection{Experimental setup}

We evaluated the effectiveness of the proposed Overlap-FedAvg by applying it to extensive models and datasets. For the base settings, the total number of the illustrative clients in our experiments was set to $10$, the max number of global epoch (a epoch usually contains thousands of iterations depending on the size of the dataset and the batch size) to train was set to $500$ and all experiments were done under a Non-IID environment in order to simulate the real-world scenario. Furthermore, the code of the experiments was developed with Pytorch and Pytorch Distributed RPC\footnote{https://pytorch.org/docs/stable/rpc.html} Framework with GLOO backend\footnote{https://github.com/facebookincubator/gloo} from scratch. NCCL backend was not used in our experiments due to its limited support for many edge devices without Nvidia GPU. To the best of our knowledge, Overlap-FedAvg is the first attempt to parallel the training phase and communication phase in the federated learning process, and it is totally capable of coping with other compression methods. As a result, in this work, we just compared the efficiency of Overlap-FedAvg with the vanilla FedAvg~\cite{mcmahan2017communication} (i.e., the metrics of FedAvg were treated as the baseline). Particularly, the following experiments shared the same experimental configurations.

\subsection{Models and Partition of Non-IID Datasets}

Several datasets were chosen to investigate the efficiency of the Overlap-FedAvg, which covered two typical application fields of deep learning. Specifically, Mnist~\cite{lecun2010mnist}, Fashion-Mnist\cite{xiao2017fashion}, EMnist\cite{cohen2017emnist}, Cifar10\cite{krizhevsky2009learning}, and Cifar100\cite{krizhevsky2009learning} were benchmark datasets in Image classification and Wikitext-2\cite{merity2016pointer} was the benchmark dataset in NLP. In order to meet the non-IID characteristic in real-world federated learning, we manually divided each of these datasets into ten Non-IID parts. The goal of this Non-IID partition was to make each part be different in both quantity and classes from other parts. The visualization of the partitioning is shown in Figure.~\ref{fig:dataset-distribution}. Some clients had plenty of training samples while some other clients only had a few, and in most clients, there are some classes that were insufficient or even absent. 

Moreover, there were six benchmark models unitized to train on these datasets: Multi-Layer Perceptron (MLP), MnistNet, CNNCifar, VGG$^R$, ResNet$^R$ and Transformer\cite{vaswani2017attention}. Among them, MLP and CNNCifar were also presented in the reference~\cite{mcmahan2017communication} to testify the effectiveness of the vanilla FedAvg, and here we utilized these two models to make a comparison between the proposed Overlap-FedAvg and the vanilla FedAvg. MnistNet, as a simplification from CNNCifar, was a standard model with only 2 convolutional layers and 2 fully connected layers without any pooling. Apart from that, the input of the MnistNet was also altered to process $1 \times 28 \times 28$ matrices instead of $3 \times 32 \times 32$ matrices. Similarly, VGG$^R$ here was also a simplified version of the original VGG11~\cite{simonyan2014very}, where all dropout layer and batch normalization layer were removed and the fully connected layers' size as well as the number of convolutional filters was reduced by half. This kind of simplification was also unitized in another state-of-art work~\cite{sattler2019robust}. Furthermore, the ResNet$^R$ model used in our experiments had its batch normalization layers removed. Besides, all other layers remained the same as the original ResNet18~\cite{he2016deep}. Finally, the basic configuration of the transformer model adopted from pytorch\footnote{\url{https://github.com/pytorch/examples/blob/master/word_language_model/model.py}} is summarized in Table \ref{tab:transformer-arch}.

\subsection{Comparison of accuracy}

\begin{figure}[h]
    \centering
    \subfloat[]{%
       \includegraphics[width=0.5\linewidth]{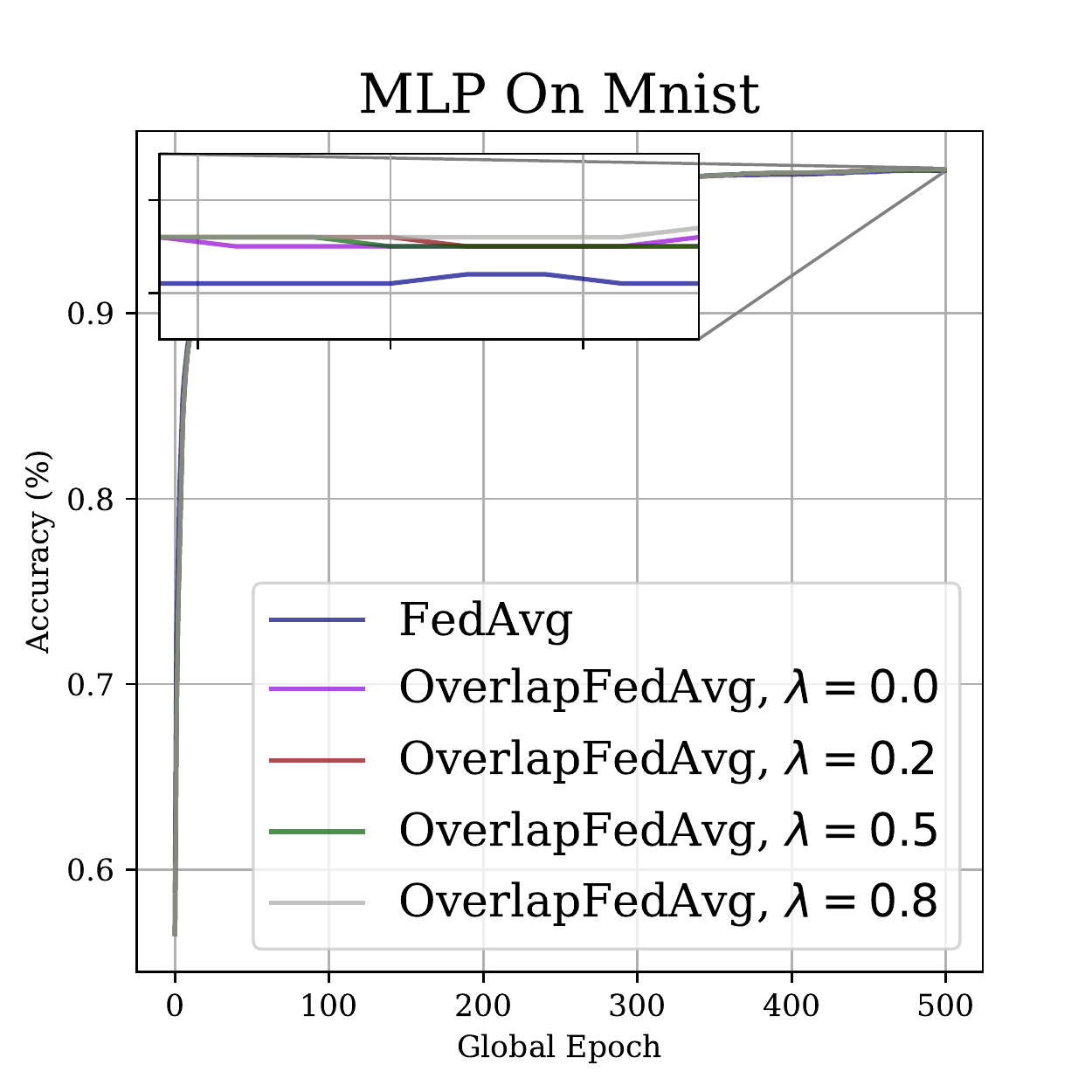}}
    \hfill
    \subfloat[]{%
        \includegraphics[width=0.5\linewidth]{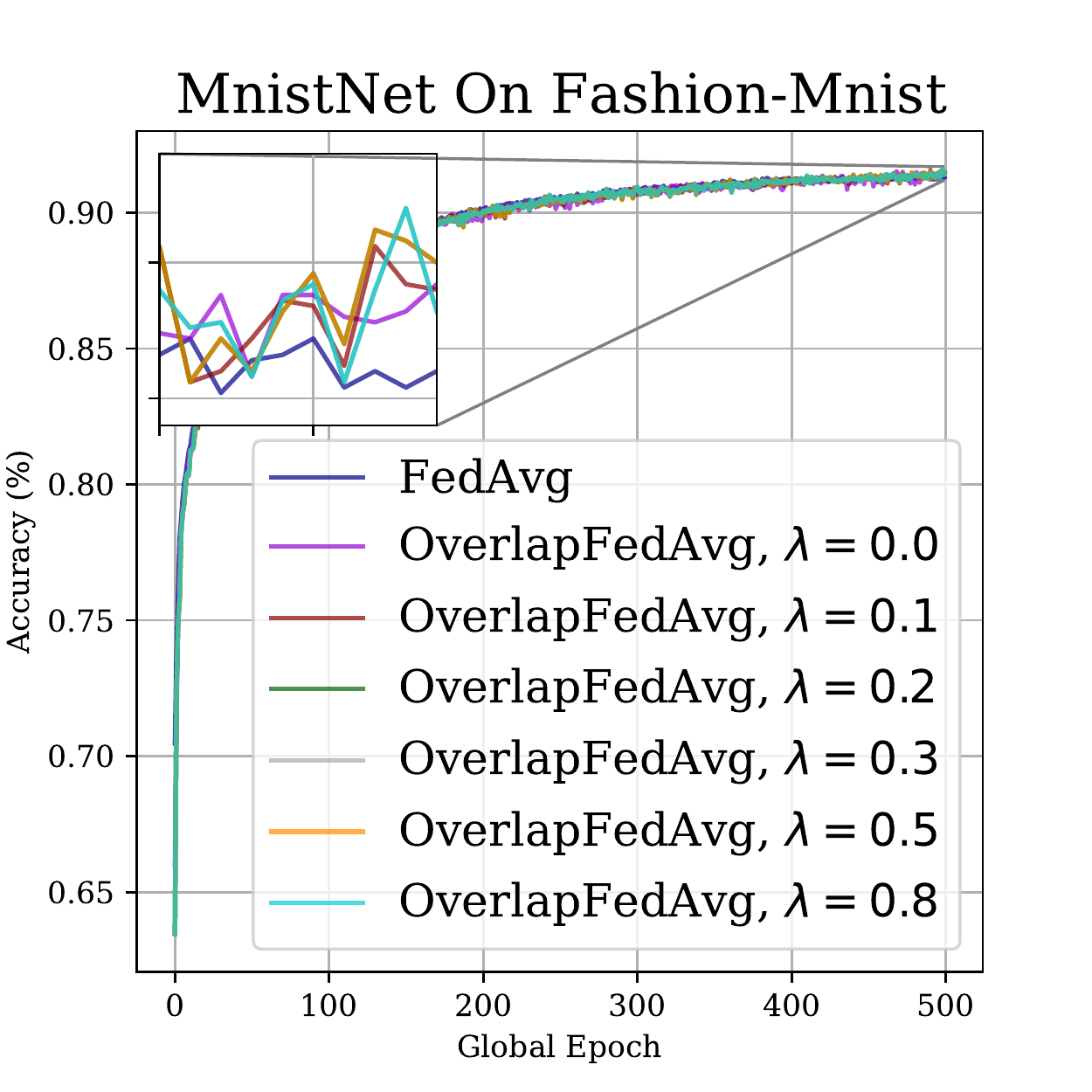}}
    \\
    \subfloat[]{%
        \includegraphics[width=0.5\linewidth]{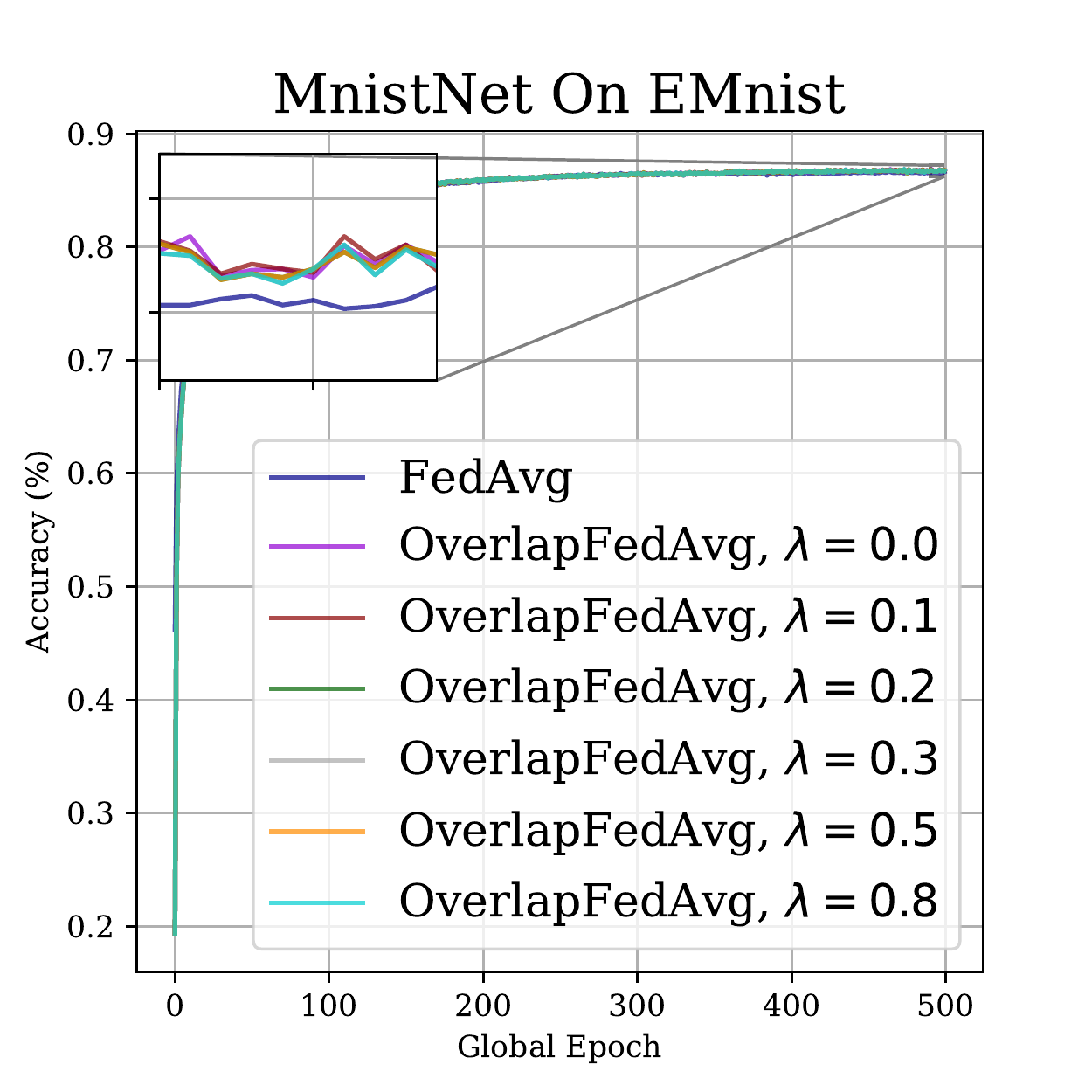}}
    \hfill
    \subfloat[]{%
        \includegraphics[width=0.5\linewidth]{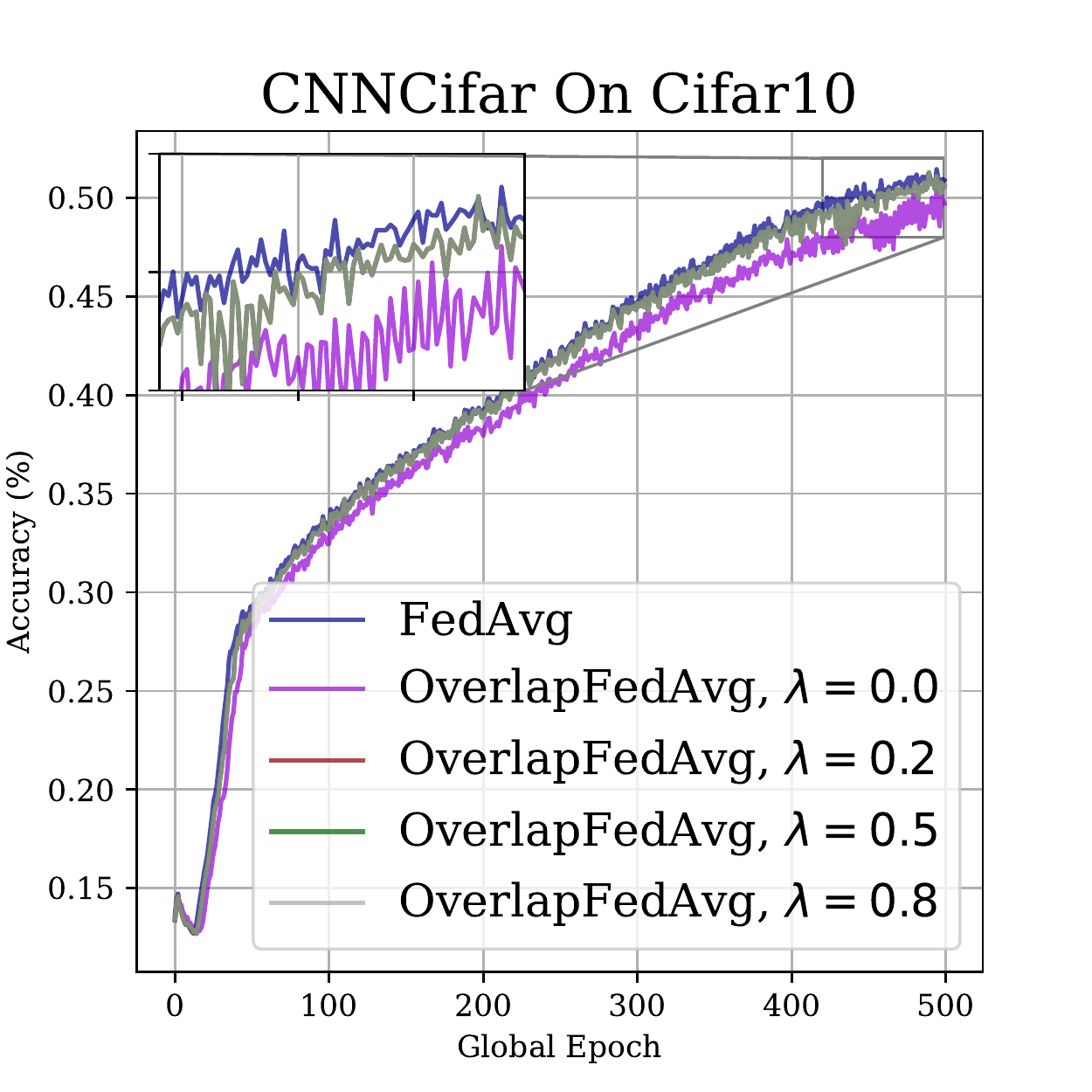}}
    \caption{(a) The accuracy curve of MLP trained on Mnist. (b) The accuracy curve of MnistNet trained on Fashion-Mnist. (c) The accuracy curve of MnistNet trained on EMnist. (d) The accuracy curve of CNNCifar trained on Cifar10.}
    \label{fig:acc-curve} 
\end{figure}

Firstly, we validated the efficacy of the hierarchical computing strategy and the data compensation mechanism in Overlap-FedAvg. Detailedly, we trained MLP on Mnist, MnistNet on Fasion-Mnist, MnistNet on EMnist and CNNCifar on Cifar10 with a learning rate $\eta = 0.001$. The number of the client iteration $E$ was set to a fixed constant $5$ for vanilla FedAvg, while the upper limit number of the client iteration was set to $5$ for Overlap-FedAvg. That is to say, in Overlap-FedAvg, although the communication interval is adaptive to the environment, it can not exceed $5$, which limits the worst performance of Overlap-FedAvg.

\begin{figure}[b]
	\centering
	\includegraphics[width=\linewidth]{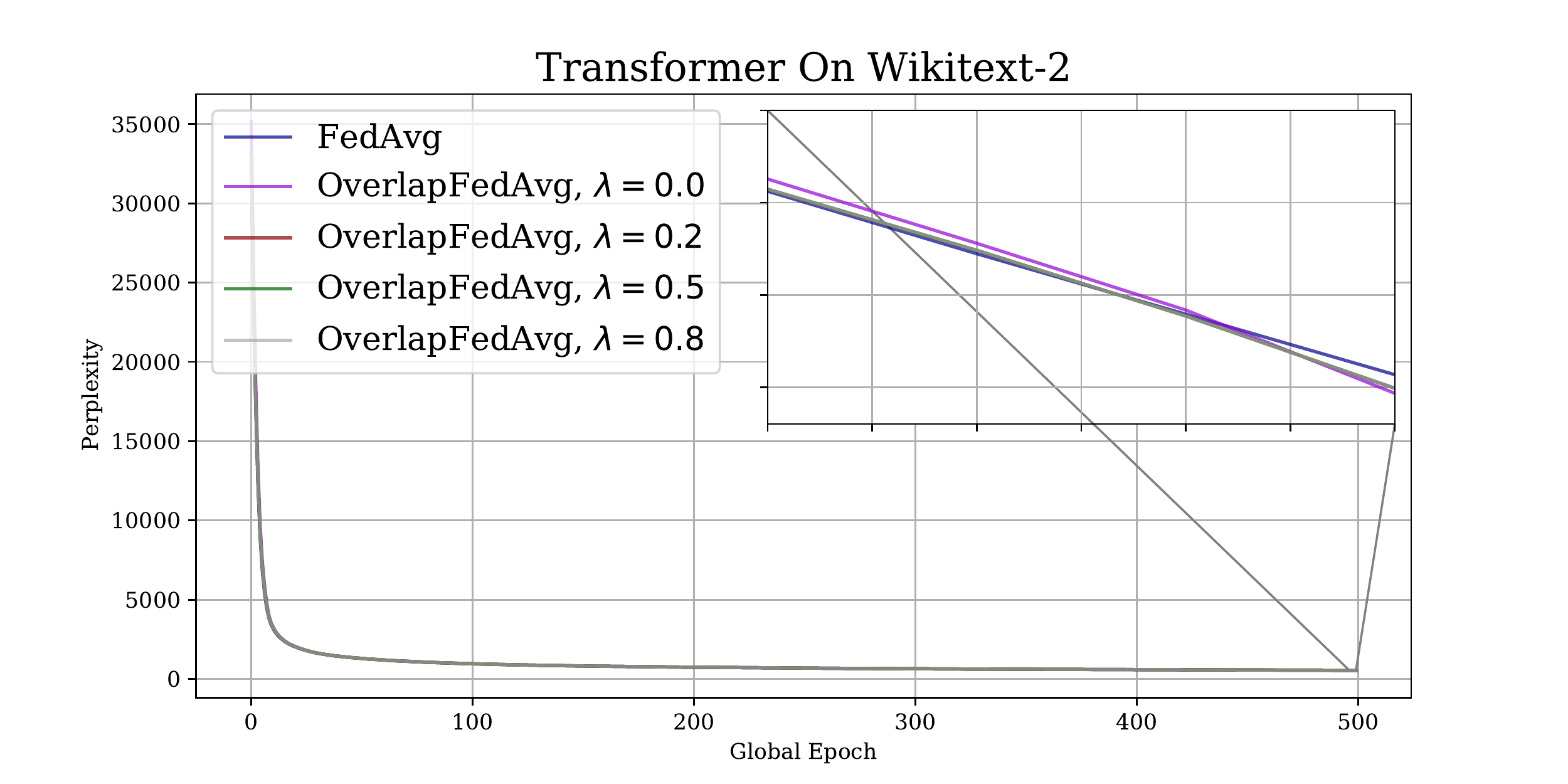}
	\caption{The perplexity curve of Transformer trained on Wikitext-2 dataset}
	\label{fig:trans-wikitext}
\end{figure}

\begin{figure}[h]
    \centering
    \subfloat[]{%
        \includegraphics[width=\linewidth]{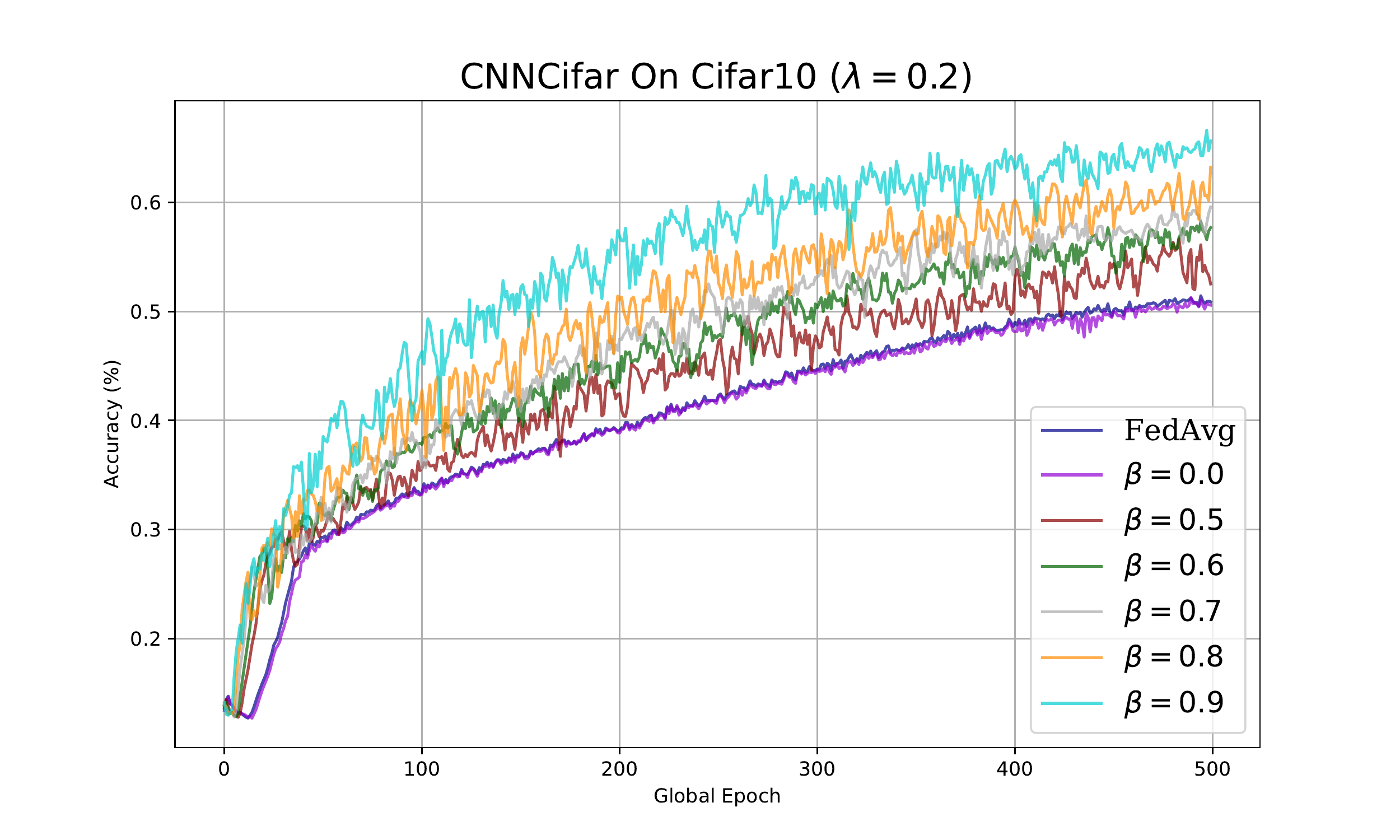}}
    \\
    \subfloat[]{%
        \includegraphics[width=\linewidth]{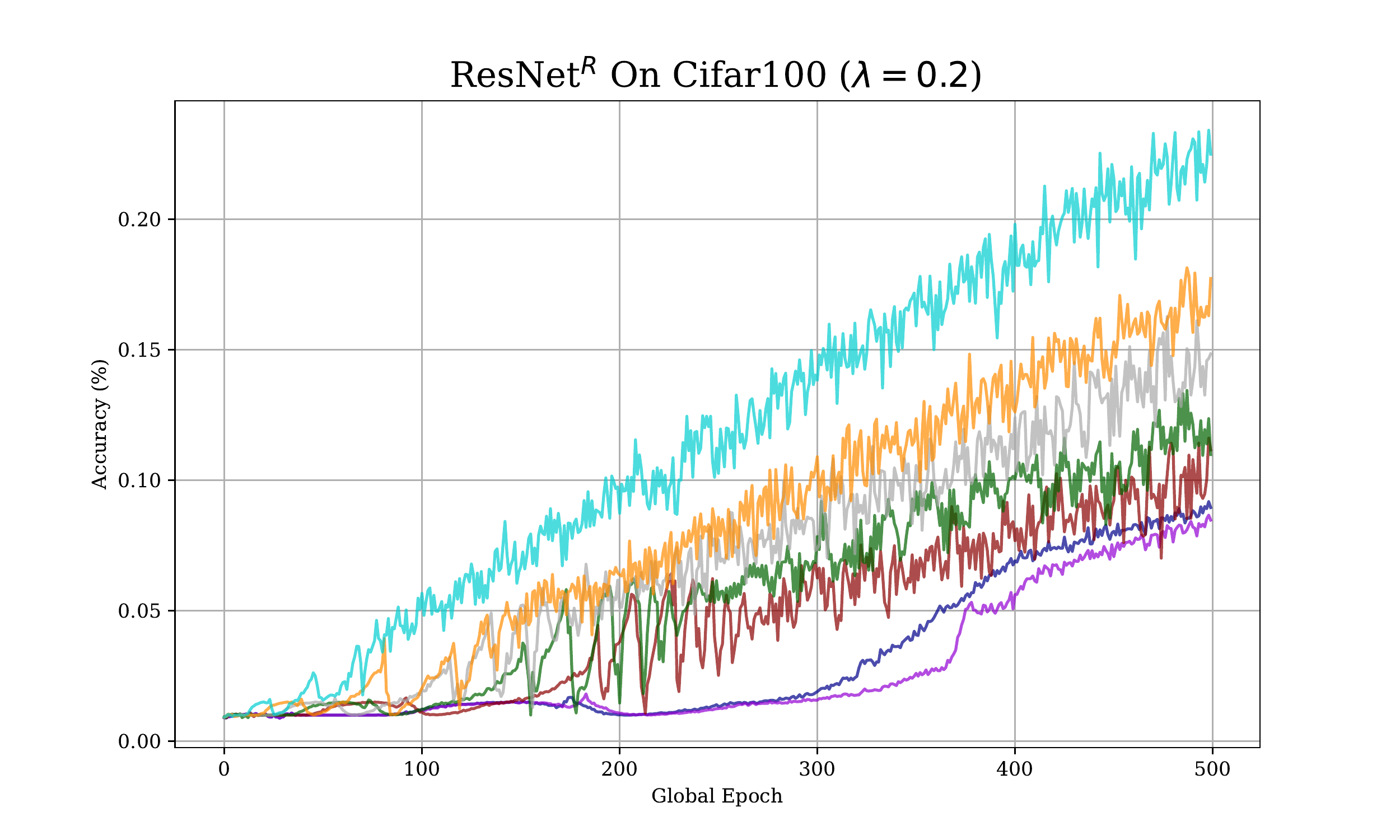}}
    \caption{(a) CNNCifar trained on Cifar10 and (b) ResNet$^R$ trained on Cifar100. For Overlap-FedAvg, $\lambda$ is fixed to $0.2$. It can be seen that the NAG implemented with compensated gradients drastically boosted the convergence speed of the federated learning.}
    \label{fig:cnncifar-mom} 
\end{figure}

% Table generated by Excel2LaTeX from sheet 'Sheet2'
\begin{table}[htb]
 \newcommand{\tabincell}[2]{\begin{tabular}{@{}#1@{}}#2\end{tabular}}
  \centering
  \captionsetup{justification=centering}
  \caption{\\\textsc{Accuracy comparison of overlap-fedavg with different $\beta$ and fixed $\lambda = 0.2$}}
    \begin{tabular}{cccccc}
    \toprule
    \toprule
    \multirow{2}[4]{*}{Model} & \multirow{2}[4]{*}{dataset} & \multirow{2}[4]{*}{$\eta$} & \multirow{2}[4]{*}{$\beta$} & \multicolumn{2}{c}{Accuracy/PPL} \\
\cmidrule{5-6}          &       &       &       & \multicolumn{1}{l}{\tabincell{c}{\textbf{50}\\{iterations}}} & \tabincell{c}{\textbf{500}\\{iterations}} \\
    \midrule
    MLP   & Mnist & 0.001 & 0.0     & 0.9317 & 0.9775 \\
    MLP   & Mnist & 0.001 & 0.1   & 0.9397 & 0.9793 \\
    MLP   & Mnist & 0.001 & 0.5   & 0.9536 & 0.9805 \\
    MLP   & Mnist & 0.001 & 0.8   & 0.9607 & 0.9781 \\
    \midrule
    MnistNet & FMNist & 0.001 & 0.0     & 0.8596 & 0.8676 \\
    MnistNet & FMNist & 0.001 & 0.5   & 0.8298 & 0.9145 \\
    \midrule
    MnistNet & Emnist & 0.001 & 0.0     & 0.8134 & 0.8676 \\
    MnistNet & Emnist & 0.001 & 0.1   & 0.7921 & 0.8651 \\
    MnistNet & Emnist & 0.001 & 0.5   & 0.8119 & 0.8647 \\
    \midrule
    CNNCifar & Cifar10 & 0.001 & 0.0     & 0.2890 & 0.5060 \\
    CNNCifar & Cifar10 & 0.001 & 0.5   & 0.2984 & 0.5252 \\
    CNNCifar & Cifar10 & 0.001 & 0.6   & 0.3088 & 0.5769 \\
    CNNCifar & Cifar10 & 0.001 & 0.7   & 0.3138 & 0.5958 \\
    CNNCifar & Cifar10 & 0.001 & 0.8   & 0.3234 & 0.6323 \\
    CNNCifar & Cifar10 & 0.001 & 0.9   & 0.3853 & 0.6564 \\
    \midrule
    VGG$^R$ & Cifar10 & 0.0001 & 0.0     & 0.1003 & 0.4247 \\
    VGG$^R$ & Cifar10 & 0.0001 & 0.5   & 0.1364 & 0.4615 \\
    \midrule
    ResNet$^R$ & Cifar10 & 0.0001 & 0.0     & 0.1382 & 0.4345 \\
    ResNet$^R$ & Cifar10 & 0.0001 & 0.5   & 0.1483 & 0.4388 \\
    ResNet$^R$ & Cifar10 & 0.0001 & 0.9   & 0.2252 & 0.5164 \\
    \midrule
    ResNet$^R$ & Cifar100 & 0.0001 & 0.0     & 0.0100 & 0.0866 \\
    ResNet$^R$ & Cifar100 & 0.0001 & 0.5   & 0.0121 & 0.1116 \\
    ResNet$^R$ & Cifar100 & 0.0001 & 0.6   & 0.0141 & 0.1354 \\
    ResNet$^R$ & Cifar100 & 0.0001 & 0.7   & 0.0150 & 0.1611 \\
    ResNet$^R$ & Cifar100 & 0.0001 & 0.8   & 0.0161 & 0.1814 \\
    ResNet$^R$ & Cifar100 & 0.0001 & 0.9   & 0.0262 & 0.2341 \\
    \midrule
    Transformer & Wikitext-2 & 0.001 & 0.0     & 1292.328 & 546.921 \\
    Transformer & Wikitext-2 & 0.001 & 0.1   & 1234.023 & 529.433 \\
    Transformer & Wikitext-2 & 0.001 & 0.5   & 1110.573 & 484.967 \\
    Transformer & Wikitext-2 & 0.001 & 0.9   & 1177.568 & 385.708 \\
    \midrule
    \multicolumn{6}{c}{$\lambda=0.2$} \\
    \bottomrule
    \bottomrule
    \end{tabular}%
  \label{tab:accuracy-beta}%
\end{table}%

The accuracy curve of these experiments is shown in Figure \ref{fig:acc-curve}. As we can see, when setting $\lambda = 0$ (i.e. do not compensate the stale data at all), the final accuracy of the Overlap-FedAvg was already surpassed vanilla FedAvg in some lightweight experiments (i.e. experiments with lightweight models or datasets). This phenomenon was caused by the lower synchronization interval in Overlap-FedAvg. Namely, these models were all relatively small, and when transmitting such a small model, the communication interval of $5$ is large, and therefore Overlap-FedAvg adaptively lowered the communication interval to a suitable value that fitted the current network and hardware condition, which proved the adaptability of the hierarchical computing strategy. Moreover, although the communication interval was automatically reconfigured to a more appropriate value, the model parameters staleness problem still existed. Therefore, we utilized the compensation method to alleviate this problem, and the introduced hyper-parameter $\lambda$ was used to adjust the compensation. As shown in Figure.~\ref{fig:acc-curve}, it is clear that a model with higher final accuracy can be obtained by fine-tuning $\lambda$.

Moreover, we expanded our experiments to some large models with complex datasets. Specifically, the VGG$^R$ was trained on Cifar10, and ResNet$^R$ was trained on Cifar10, Cifar100 with $\eta=0.0001$, respectively. The results of all the experiments were described in Table~\ref{tab:accuracy-lambda}. With the larger model size, the communication interval can not be decreased by Overlap-FedAvg for higher communication efficiency as well as better final accuracy. However, the time originally used for model uploading and model downloading can at least be saved by Overlap-FedAvg, which could be still a large amount in some cases. From the Table, we can see that when the $\lambda$ was configured to $0$, unlike previous lightweight experiments, Overlap-FedAvg had a lower final accuracy compared to vanilla FedAvg due to the stale model parameters problem. However, by adjusting the $\lambda$ to properly compensate the data, the gap between Overlap-FedAvg and vanilla FedAvg can be greatly reduced.

Furthermore, we validated the Overlap-FedAvg's feasibility by applying it to gradients sensitive Natural Language Processing (NLP) tasks. Specifically, we utilized transformer to be trained on the wikitext-2 dataset, and used perplexity to denote the model's performance (the lower the better), as drawn in Figure \ref{fig:trans-wikitext}. With the additional assistance of the hierarchical computing strategy and the data compensation mechanism, Overlap-FedAvg with the stale problem was also capable of achieving nearly the same convergence speed with respect to the global iteration, and reaching a very similar final accuracy compared to the vanilla FedAvg.

Then, to evaluate the effectiveness of the NAG algorithm implemented in Overlap-FedAvg, we fixed $\lambda=0.2$ and trained several DNN models with different $\beta$, which is illustrated in Figure \ref{fig:cnncifar-mom} and Table \ref{tab:accuracy-beta}. From the Figure \ref{fig:cnncifar-mom}, we can see that when the NAG was enabled, Overlap-FedAvg significantly outperformed the convergence speed of vanilla FedAvg, which strongly verified the usability of this acceleration. Consequently, considering the implementation of this acceleration was so intuitive and easy to implement, many other acceleration methods originated in SGD, be it Adam\cite{kingma2014adam} or RMSProp\cite{Tieleman2012}, should be applied to Overlap-FedAvg with minimal effects.

\subsection{Comparison of Training Speed}

In this section, we compared the training speed of the Overlap-FedAvg and the vanilla FedAvg to demonstrate the communication efficiency of the Overlap-FedAvg framework. Specifically, the average wall-clock time for one iteration of federated learning in all above-mentioned experiments is presented in Figure~\ref{fig:wall-clock-compare} and Table~\ref{tab:average-train-time}.

From the Figure~\ref{fig:wall-clock-compare}, it can be seen that in all experiments, Overlap-FedAvg successfully lowered the wall-clock time for each iteration training. Specifically, when training lightweight models (e.g. MLP, MnistNet or CNNCifar), compared to vanilla FedAvg, Overlap-FedAvg approximately saveed 10\% of the training time. When it came to train some slightly heavier models (e.g. VGG$^R$, ResNet$^R$, Transformer), Overlap-FedAvg significantly boosted the training process by at most 40\%. Consequently, since the size of the parameters from MLP to Transformer is increasing, which is documented in Table \ref{tab:average-train-time}, we concluded that Overlap-FedAvg saveed more training time with heavier models. 

\begin{figure}[htp]
	\centering
	\includegraphics[width=\linewidth]{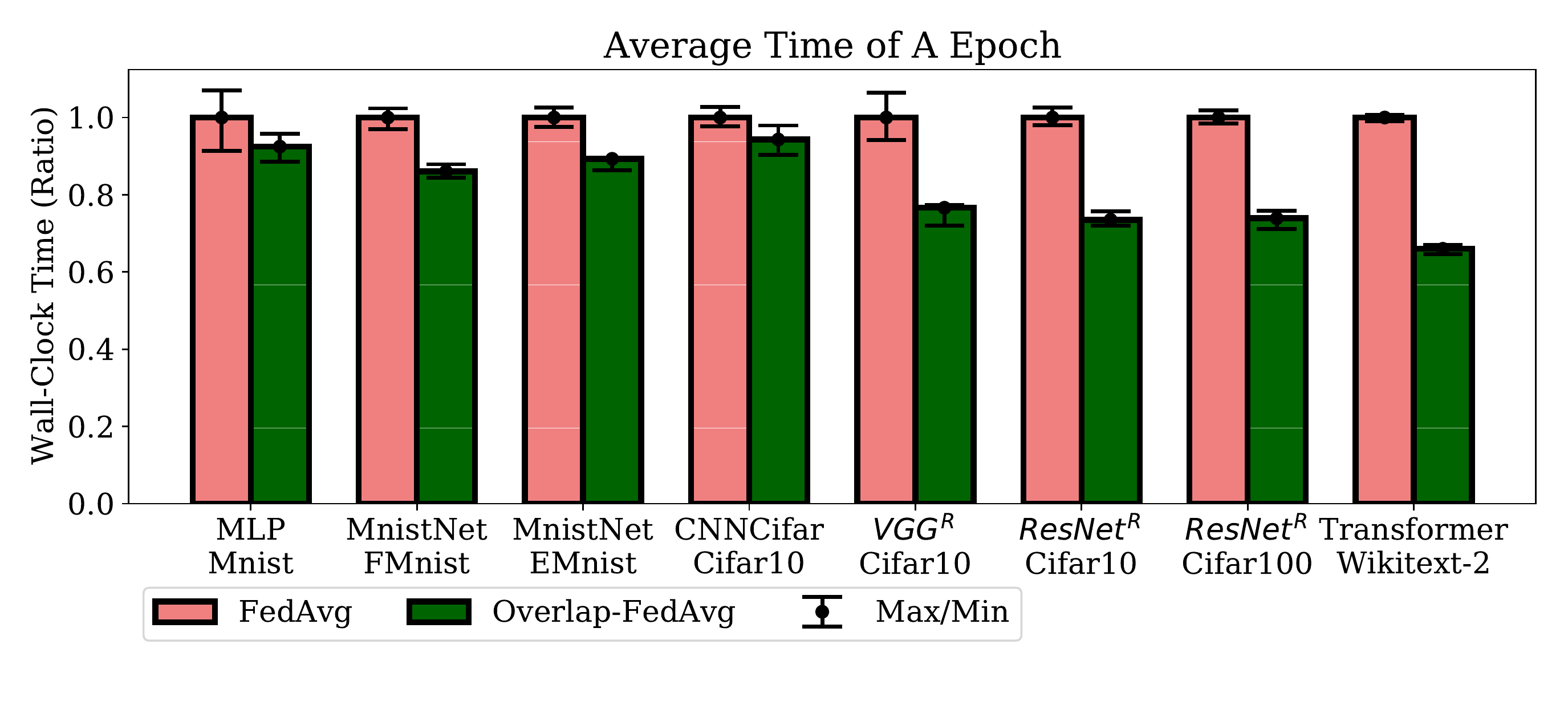}
	\caption{The average wall-clock time for one iteration of federated learning in MLP on Mnist, MnistNet on FMnist, MnistNet on EMnist, CNNCifar on Cifar10, VGG$^R$ on Cifar10, ResNet$^R$ on Cifar10, ResNEt$^R$ on Cifar100 and Transformer on Wikitext-2. We can see that the Overlap-FedAvg saved more time with the increase of the model size.}
	\label{fig:wall-clock-compare}
\end{figure}

\begin{figure}[htb]
	\centering
	\includegraphics[width=\linewidth]{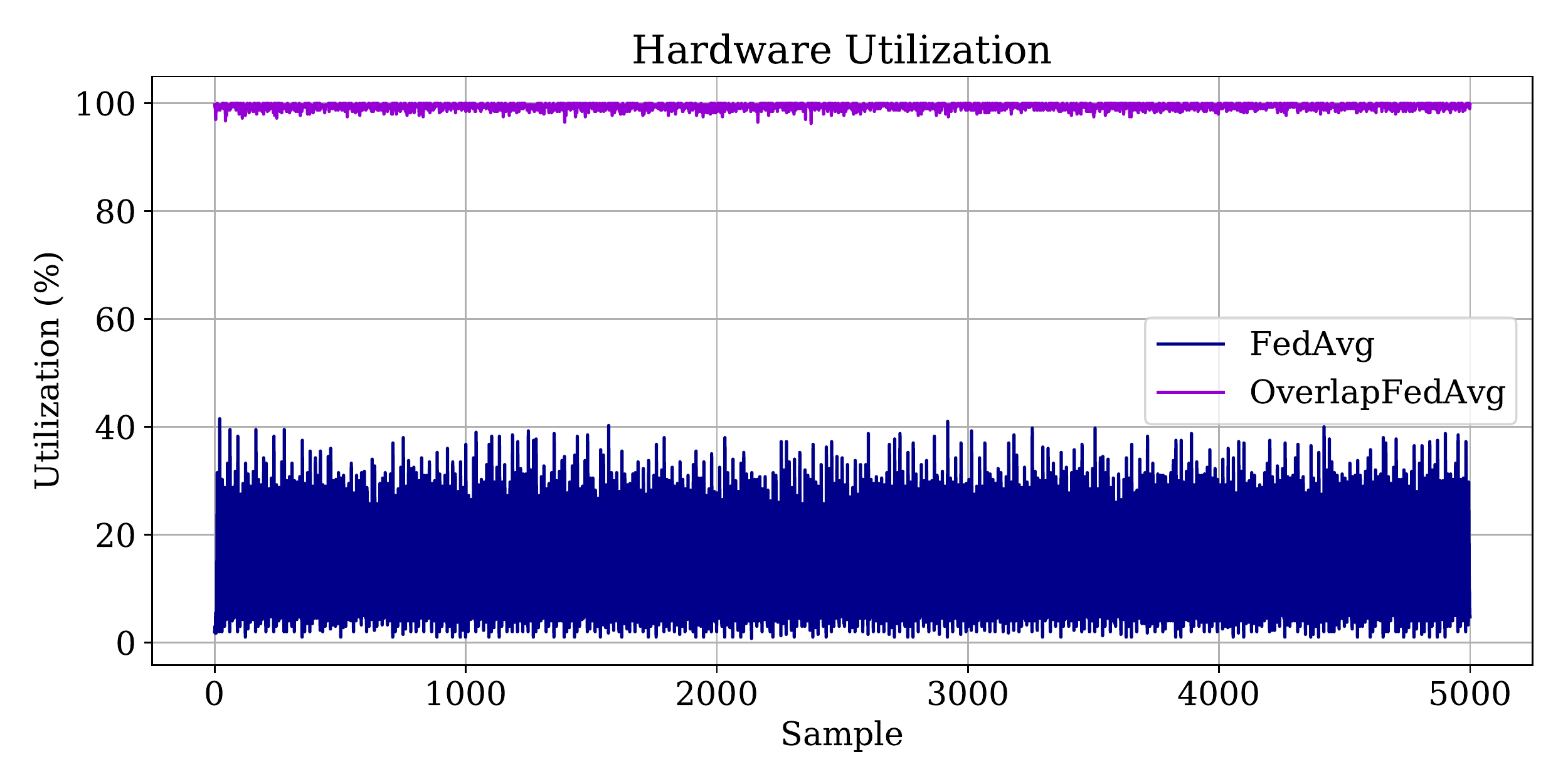}
	\caption{The hardware utilization rate during the federated learning process. Overlap-FedAvg had a much higher utilization rate compared to vanilla FedAvg.}
	\label{fig:util-compare}
\end{figure}

We also compared the efficiency of the proposed Overlap-FedAvg from the hardware's perspective. Specifically, we recorded the utilization rate of the computation during the federated learning process, and made a visualization of it after the learning, which is shown in Figure \ref{fig:util-compare}. It can be seen that Overlap-FedAvg achieveed a much higher hardware utilization rate compared to FedAvg due to the continuous local training strategy. In fact, Overlap-FedAvg almost made the best use of the computational resource as the hardware utilization rate of the Overlap-FedAvg was nearly fixed at 100\%, which further validates the effectiveness of Overlap-FedAvg.

% Table generated by Excel2LaTeX from sheet 'Sheet3'
\begin{table}[htbp]
  \centering
  \captionsetup{justification=centering}
  \caption{\\\textsc{The average wall-clock time for one iteration of federated learning}}
    \begin{tabular}{ccccc}
    \toprule
    \toprule
    \multirow{2}[4]{*}{Model} & \multirow{2}[4]{*}{dataset} & \multirow{2}[4]{*}{Parameters} & \multicolumn{2}{c}{Time / Iteration (Second)} \\
\cmidrule{4-5}          &       &       & FedAvg & Overlap-FedAvg \\
    \midrule
    MLP   & Mnist & 199,210 & 31.2  & \textbf{28.85}~($\downarrow$7.53\%) \\
    MnistNet & FMnist & 1,199,882 & 32.96 & \textbf{28.31}~($\downarrow$14.11\%) \\
    MnistNet & Emnist & 1,199,882 & 47.19 & \textbf{42.15}~($\downarrow$10.68\%) \\
    CNNCifar & Cifar10 & 878,538 & 48.07 & \textbf{45.33}~($\downarrow$5.70\%) \\
    VGG$^R$ & Cifar10 & 2,440,394 & 64.4  & \textbf{49.33}~($\downarrow$23.40\%) \\
    ResNet$^R$ & Cifar10 & 11,169,162 & 156.88 & \textbf{115.31}~($\downarrow$26.50\%) \\
     ResNet$^R$ & Cifar100 & 11,169,162 & 156.02 & \textbf{115.3}~($\downarrow$26.10\%) \\
    Transformer & Wikitext-2 & 13,828,478 & 133.19 & \textbf{87.9}~($\downarrow$34.0\%) \\
    \bottomrule
    \bottomrule
    \end{tabular}%
  \label{tab:average-train-time}%
\end{table}%

\section{Discussion}
\label{section:six}
\subsection{Privacy Security of OverLep-FedAvg}

The main target for federated learning is to train the DNN model without revealing any local data to others, including the central server who coordinates the entire training process. However, some researchers played the role of attackers and tried to find ways to gain private information from the intermediate data (e.g. model weights, gradients.) generated during the training process, which was regarded as desensitized. Recently, some work~\cite{zhu2019deep, zhao2020idlg} pointed out that gradients can be utilized to inversely generate the training samples, which to some extent proved false of the previous idea that gradients are not sensitive data in deep learning. In other words, the transmission of the exposed weight in federated learning may result in the privacy leakage of the decentralized clients~\cite{2020arXiv200302133L}.

In vanilla FedAvg, the learning rate of each client was configured by the central server, and the local DNN weights trained by various clients were accumulated to compute the global weights on the central server. Therefore, if the central server is "curious", it is totally feasible for it to compute the clients' gradients from their weights with the known learning rate, and then extracts the hidden training data with the method proposed in references~\cite{zhu2019deep, zhao2020idlg}, which is no doubt detrimental to privacy security of the federated learning. Fortunately, since the learning rate is not really involved when updating the global model in the vanilla FedAvg, this issue can be resolved by letting clients generate their own learning rate at the beginning of the training, since currently the data restoring techniques in \cite{zhu2019deep, zhao2020idlg} was basically faking a input and label, passing them into the model, getting a fake gradients, utilizing back-propagation to calculate the loss between the real gradients from the clients and the fake gradients, and then using SGD to optimize the input and label (i.e. $\mathcal{L}(\cdot) = \left|\left| \nabla F_{fake} - \nabla F_{real} \right|\right|$). Thus, by adding a unknown learning rate, the SGD has to optimize two multiplied parameters (i.e. $\mathcal{L}(\cdot) = \left|\left| \eta \nabla F_{fake} - w_{real} \right|\right|$), which is theoretically impossible to optimize. In fact, we think it is very crucial to make sure the clients' learning rate is not known to the central server. Therefore, as the Overlap-FedAvg needs gradients to calculate the compensation, a question emerges: will the Overlap-FedAvg work safely without the central server knowing the clients' learning rate?

Yes, it will. Overlap-FedAvg indeed needs gradients, but it does not need the \textit{real} gradients since the learning rate is a constant that can be extracted. In other words, no matter what the real learning rate is in clients, the central server can just assume it is $1.0$, computing scaled gradients and utilizing the scaled gradients to calculate the compensation. Formally, presuming the clients' learning rate is $\eta_c$, and therefore the scaled gradients $\nabla F(w_{t-1})_{scaled} = \eta_c \nabla F(w_{t-1})_{real}$. When compensating, following the algorithm we can derive the compensated gradients as Equation~\ref{compensated_gradient}:

\begin{align}\label{compensated_gradient}
\begin{split}
    &\nabla F(w_{t-1})_{scaled}^{ah} = \nabla F(w_{t-1})_{scaled}^{ah} \\
    &\qquad + \lambda \nabla F(w_{t-1})_{scaled}^{ah} \odot \nabla F(w_{t-1})_{scaled}^{ah} \odot (w_{t} - w_{t-1}) \\
    &= \eta_c \nabla F(w_{t-1})_{real} \\
    &\qquad + \lambda \eta_c^2 \nabla F(w_{t-1})_{real} \odot \nabla F(w_{t-1})_{real} \odot (w_{t} - w_{t-1}) \\
    &= \eta_c \nabla F(w_{t-1})_{real}^{ah},
\end{split}
\end{align}

which suggests that the compensated scaled gradients is essentially equivalent to the compensated real gradients multiplied by the unknown clients' learning rate $\eta_c$, and can be directly brought into the SGD updating rule to update the global model. In summary, the privacy security of Overlap-FedAvg is guaranteed without leaking the learning rate to the central server. 

\section{Conclusions and future work}
Learning from massive data stored in different locations is essential in many scenarios. To efficiently accomplish the learning, federated learning was introduced to provide a practical and privacy-preserving approach for DNN training. However, federated learning introduces massive communication overhead resulted from the heavy communication of DNN weights and low bandwidth.

In this paper, we proposed a novel federated learning framework, named Overlap-FedAvg, to achieve communication-efficient federated learning from the structural perspective. Comprehensively, Overlap-FedAvg firstly parallels the training phase with communication phase to allow continuous local model training without any blocking caused by the frequent communicating. Then, Overlap-FedAvg introduces a data compensation mechanism to resolve the stale data problem brought by the parallelism, and ensures the same convergence rate in comparison with the vanilla FedAvg under the same circumstance. Finally, an intuitive and easy-to-implement NAG algorithm is described, which could be coped with many other federated learning algorithms. A theoretical analysis was provided to guarantee the convergence of the Overlap-FedAvg. Extensive experiments also demonstrated that Overlap-FedAvg with only parallelism and data compensation is already capable to considerably speed up the federated learning process while maintaining nearly the same model performance (i.e. final accuracy) compared to FedAvg. Furthermore, with the NAG enabled, Overlap-FedAvg massively surpassed the vanilla FedAvg in accuracy metric concerning both the number of training iteration and the wall-clock time, which strongly validate the feasibility and effectiveness of our proposed method, especially when the model is relatively big and the network bandwidth of clients is slow or unstable.

Despite the good efficacy of Overlap-FedAvg, it inevitably introduced a new hyper-parameter $\lambda$ in the data compensation mechanism to properly estimate the Hessian matrix of model weights by controlling the variance between estimated values and real values. Consequently, a comprehensive sampling of the environment is required to obtain an optimal $\lambda$, which is costly for both time and computational resources. In future work, we would like to investigate methods capable of adaptively selecting $\lambda$ by evolutionary algorithms or reinforcement learning to further improve the practicability of the Overlap-FedAvg.

% or


% \appendices

% % you can choose not to have a title for an appendix
% % if you want by leaving the argument blank
% \section{}

% % use section* for acknowledgment
% \section*{Acknowledgment}


\ifCLASSOPTIONcaptionsoff
  \newpage
\fi



\bibliographystyle{IEEEtran}

\begin{thebibliography}{10}
\providecommand{\url}[1]{#1}
\csname url@samestyle\endcsname
\providecommand{\newblock}{\relax}
\providecommand{\bibinfo}[2]{#2}
\providecommand{\BIBentrySTDinterwordspacing}{\spaceskip=0pt\relax}
\providecommand{\BIBentryALTinterwordstretchfactor}{4}
\providecommand{\BIBentryALTinterwordspacing}{\spaceskip=\fontdimen2\font plus
\BIBentryALTinterwordstretchfactor\fontdimen3\font minus
  \fontdimen4\font\relax}
\providecommand{\BIBforeignlanguage}[2]{{%
\expandafter\ifx\csname l@#1\endcsname\relax
\typeout{** WARNING: IEEEtran.bst: No hyphenation pattern has been}%
\typeout{** loaded for the language `#1'. Using the pattern for}%
\typeout{** the default language instead.}%
\else
\language=\csname l@#1\endcsname
\fi
#2}}
\providecommand{\BIBdecl}{\relax}
\BIBdecl

\bibitem{Lecun2015Deep}
Y.~Lecun, Y.~Bengio, and G.~Hinton, ``Deep learning,'' \emph{Nature}, vol. 521,
  no. 7553, p. 436, 2015.

\bibitem{Brown2020LanguageMA}
T.~B. Brown, B.~P. Mann, N.~Ryder, M.~Subbiah, J.~Kaplan, P.~Dhariwal,
  A.~Neelakantan, P.~Shyam, G.~Sastry, A.~Askell, S.~Agarwal, A.~Herbert-Voss,
  G.~Kr{\"u}ger, T.~Henighan, R.~Child, A.~Ramesh, D.~M. Ziegler, J.~Wu,
  C.~Winter, C.~Hesse, M.~Chen, E.~J. Sigler, M.~Litwin, S.~Gray, B.~Chess,
  J.~Clark, C.~Berner, S.~McCandlish, A.~Radford, I.~Sutskever, and D.~Amodei,
  ``Language models are few-shot learners,'' 2020.

\bibitem{devlin2018bert}
J.~Devlin, M.-W. Chang, K.~Lee, and K.~Toutanova, ``Bert: Pre-training of deep
  bidirectional transformers for language understanding,'' \emph{arXiv preprint
  arXiv:1810.04805}, 2018.

\bibitem{krizhevsky2017imagenet}
A.~Krizhevsky, I.~Sutskever, and G.~E. Hinton, ``Imagenet classification with
  deep convolutional neural networks,'' \emph{Communications of the ACM},
  vol.~60, no.~6, pp. 84--90, 2017.

\bibitem{simonyan2014very}
K.~Simonyan and A.~Zisserman, ``Very deep convolutional networks for
  large-scale image recognition,'' \emph{arXiv preprint arXiv:1409.1556}, 2014.

\bibitem{edge-computing}
Y.~{Mao}, C.~{You}, J.~{Zhang}, K.~{Huang}, and K.~B. {Letaief}, ``A survey on
  mobile edge computing: The communication perspective,'' \emph{IEEE
  Communications Surveys Tutorials}, vol.~19, no.~4, pp. 2322--2358, 2017.

\bibitem{mcmahan2017communication}
H.~McMahan, E.~Moore, D.~Ramage, S.~Hampson, and B.~A. y~Arcas,
  ``Communication-efficient learning of deep networks from decentralized
  data,'' in \emph{AISTATS}, 2017.

\bibitem{Demystifying-DDL}
T.~Ben-Nun and T.~Hoefler, ``Demystifying parallel and distributed deep
  learning: An in-depth concurrency analysis,'' \emph{ACM Computing Surveys},
  vol.~52, 02 2018.

\bibitem{2019arXiv191204977K}
P.~{Kairouz}, H.~B. {McMahan}, B.~{Avent}, A.~{Bellet}, M.~{Bennis}, A.~{Nitin
  Bhagoji}, K.~{Bonawitz}, Z.~{Charles}, G.~{Cormode}, R.~{Cummings}, R.~G.~L.
  {D'Oliveira}, S.~{El Rouayheb}, D.~{Evans}, J.~{Gardner}, Z.~{Garrett},
  A.~{Gasc{\'o}n}, B.~{Ghazi}, P.~B. {Gibbons}, M.~{Gruteser}, Z.~{Harchaoui},
  C.~{He}, L.~{He}, Z.~{Huo}, B.~{Hutchinson}, J.~{Hsu}, M.~{Jaggi},
  T.~{Javidi}, G.~{Joshi}, M.~{Khodak}, J.~{Kone{\v{c}}n{\'y}}, A.~{Korolova},
  F.~{Koushanfar}, S.~{Koyejo}, T.~{Lepoint}, Y.~{Liu}, P.~{Mittal},
  M.~{Mohri}, R.~{Nock}, A.~{{\"O}zg{\"u}r}, R.~{Pagh}, M.~{Raykova}, H.~{Qi},
  D.~{Ramage}, R.~{Raskar}, D.~{Song}, W.~{Song}, S.~U. {Stich}, Z.~{Sun},
  A.~{Theertha Suresh}, F.~{Tram{\`e}r}, P.~{Vepakomma}, J.~{Wang}, L.~{Xiong},
  Z.~{Xu}, Q.~{Yang}, F.~X. {Yu}, H.~{Yu}, and S.~{Zhao}, ``{Advances and Open
  Problems in Federated Learning},'' \emph{arXiv e-prints}, p.
  arXiv:1912.04977, Dec. 2019.

\bibitem{huang2017densely}
G.~Huang, Z.~Liu, L.~Van Der~Maaten, and K.~Q. Weinberger, ``Densely connected
  convolutional networks,'' in \emph{Proceedings of the IEEE conference on
  computer vision and pattern recognition}, 2017, pp. 4700--4708.

\bibitem{xie2017aggregated}
S.~Xie, R.~Girshick, P.~Doll{\'a}r, Z.~Tu, and K.~He, ``Aggregated residual
  transformations for deep neural networks,'' in \emph{Proceedings of the IEEE
  conference on computer vision and pattern recognition}, 2017, pp. 1492--1500.

\bibitem{sattler2019sparse}
F.~Sattler, S.~Wiedemann, K.-R. M{\"u}ller, and W.~Samek, ``Sparse binary
  compression: Towards distributed deep learning with minimal communication,''
  in \emph{2019 International Joint Conference on Neural Networks
  (IJCNN)}.\hskip 1em plus 0.5em minus 0.4em\relax IEEE, 2019, pp. 1--8.

\bibitem{konevcny2016federated}
J.~Kone{\v{c}}n{\`y}, H.~B. McMahan, F.~X. Yu, P.~Richt{\'a}rik, A.~T. Suresh,
  and D.~Bacon, ``Federated learning: Strategies for improving communication
  efficiency,'' \emph{arXiv preprint arXiv:1610.05492}, 2016.

\bibitem{seide20141}
F.~Seide, H.~Fu, J.~Droppo, G.~Li, and D.~Yu, ``1-bit stochastic gradient
  descent and its application to data-parallel distributed training of speech
  dnns,'' in \emph{Fifteenth Annual Conference of the International Speech
  Communication Association}, 2014.

\bibitem{bernstein2018signsgd}
J.~Bernstein, Y.-X. Wang, K.~Azizzadenesheli, and A.~Anandkumar, ``signsgd:
  Compressed optimisation for non-convex problems,'' \emph{arXiv preprint
  arXiv:1802.04434}, 2018.

\bibitem{alistarh2017qsgd}
D.~Alistarh, D.~Grubic, J.~Li, R.~Tomioka, and M.~Vojnovic, ``Qsgd:
  Communication-efficient sgd via gradient quantization and encoding,'' in
  \emph{Advances in Neural Information Processing Systems}, 2017, pp.
  1709--1720.

\bibitem{lin2017deep}
Y.~Lin, S.~Han, H.~Mao, Y.~Wang, and W.~J. Dally, ``Deep gradient compression:
  Reducing the communication bandwidth for distributed training,'' \emph{arXiv
  preprint arXiv:1712.01887}, 2017.

\bibitem{wangni2018gradient}
J.~Wangni, J.~Wang, J.~Liu, and T.~Zhang, ``Gradient sparsification for
  communication-efficient distributed optimization,'' in \emph{Advances in
  Neural Information Processing Systems}, 2018, pp. 1299--1309.

\bibitem{NAG}
Y.~Nesterov, ``A method for unconstrained convex minimization problem with the
  rate of convergence,'' \emph{Doklady AN SSSR}, vol. 269, pp. 543--547, 01
  1983.

\bibitem{yang2020federated}
Z.~Yang, W.~Bao, D.~Yuan, N.~H. Tran, and A.~Y. Zomaya, ``Federated learning
  with nesterov accelerated gradient momentum method,'' 2020.

\bibitem{kingma2014adam}
D.~P. Kingma and J.~Ba, ``Adam: A method for stochastic optimization,''
  \emph{arXiv preprint arXiv:1412.6980}, 2014.

\bibitem{Tieleman2012}
T.~Tieleman and G.~Hinton, ``{Lecture 6.5---RmsProp: Divide the gradient by a
  running average of its recent magnitude},'' COURSERA: Neural Networks for
  Machine Learning, 2012.

\bibitem{Large-scale-DDN}
J.~Dean, G.~S. Corrado, R.~Monga, K.~Chen, M.~Devin, Q.~V. Le, M.~Z. Mao,
  M.~Ranzato, A.~Senior, P.~Tucker, K.~Yang, and A.~Y. Ng, ``Large scale
  distributed deep networks,'' in \emph{Proceedings of the 25th International
  Conference on Neural Information Processing Systems - Volume 1}, ser.
  NIPS'12.\hskip 1em plus 0.5em minus 0.4em\relax Red Hook, NY, USA: Curran
  Associates Inc., 2012, p. 1223–1231.

\bibitem{Ea-based-NAS}
Q.~{Ye}, Y.~{Sun}, J.~{Zhang}, and J.~C. {Lv}, ``A distributed framework for
  ea-based nas,'' \emph{IEEE Transactions on Parallel and Distributed Systems},
  pp. 1--1, 2020.

\bibitem{PSO-PS}
Q.~{Ye}, Y.~{Han}, Y.~{Sun}, and J.~{Lv}, ``Pso-ps:parameter synchronization
  with particle swarm optimization for distributed training of deep neural
  networks,'' in \emph{2020 International Joint Conference on Neural Networks
  (IJCNN)}, 2020, pp. 1--8.

\bibitem{bonawitz2017practical}
K.~Bonawitz, V.~Ivanov, B.~Kreuter, A.~Marcedone, H.~B. McMahan, S.~Patel,
  D.~Ramage, A.~Segal, and K.~Seth, ``Practical secure aggregation for
  privacy-preserving machine learning,'' in \emph{Proceedings of the 2017 ACM
  SIGSAC Conference on Computer and Communications Security}, 2017, pp.
  1175--1191.

\bibitem{hardy2017private}
S.~Hardy, W.~Henecka, H.~Ivey-Law, R.~Nock, G.~Patrini, G.~Smith, and
  B.~Thorne, ``Private federated learning on vertically partitioned data via
  entity resolution and additively homomorphic encryption,'' \emph{arXiv
  preprint arXiv:1711.10677}, 2017.

\bibitem{abadi2016deep}
M.~Abadi, A.~Chu, I.~Goodfellow, H.~B. McMahan, I.~Mironov, K.~Talwar, and
  L.~Zhang, ``Deep learning with differential privacy,'' in \emph{Proceedings
  of the 2016 ACM SIGSAC Conference on Computer and Communications Security},
  2016, pp. 308--318.

\bibitem{zhao2018federated}
Y.~Zhao, M.~Li, L.~Lai, N.~Suda, D.~Civin, and V.~Chandra, ``Federated learning
  with non-iid data,'' \emph{arXiv preprint arXiv:1806.00582}, 2018.

\bibitem{sattler2019robust}
F.~Sattler, S.~Wiedemann, K.-R. M{\"u}ller, and W.~Samek, ``Robust and
  communication-efficient federated learning from non-iid data,'' \emph{IEEE
  transactions on neural networks and learning systems}, 2019.

\bibitem{li2019convergence}
X.~Li, K.~Huang, W.~Yang, S.~Wang, and Z.~Zhang, ``On the convergence of fedavg
  on non-iid data,'' \emph{arXiv preprint arXiv:1907.02189}, 2019.

\bibitem{qu2020federated}
Z.~Qu, K.~Lin, J.~Kalagnanam, Z.~Li, J.~Zhou, and Z.~Zhou, ``Federated
  learning's blessing: Fedavg has linear speedup,'' \emph{arXiv preprint
  arXiv:2007.05690}, 2020.

\bibitem{karimireddy2019error}
S.~P. Karimireddy, Q.~Rebjock, S.~U. Stich, and M.~Jaggi, ``Error feedback
  fixes signsgd and other gradient compression schemes,'' \emph{arXiv preprint
  arXiv:1901.09847}, 2019.

\bibitem{danalis2005transformations}
A.~Danalis, K.-Y. Kim, L.~Pollock, and M.~Swany, ``Transformations to parallel
  codes for communication-computation overlap,'' in \emph{SC'05: Proceedings of
  the 2005 ACM/IEEE conference on Supercomputing}.\hskip 1em plus 0.5em minus
  0.4em\relax IEEE, 2005, pp. 58--58.

\bibitem{quinn1996utility}
M.~J. Quinn and P.~J. Hatcher, ``On the utility of communication--computation
  overlap in data-parallel programs,'' \emph{Journal of Parallel and
  Distributed Computing}, vol.~33, no.~2, pp. 197--204, 1996.

\bibitem{marjanovic2010overlapping}
V.~Marjanovi{\'c}, J.~Labarta, E.~Ayguad{\'e}, and M.~Valero, ``Overlapping
  communication and computation by using a hybrid mpi/smpss approach,'' in
  \emph{Proceedings of the 24th acm International Conference on
  Supercomputing}, 2010, pp. 5--16.

\bibitem{zheng2017asynchronous}
S.~Zheng, Q.~Meng, T.~Wang, W.~Chen, N.~Yu, Z.-M. Ma, and T.-Y. Liu,
  ``Asynchronous stochastic gradient descent with delay compensation,'' in
  \emph{International Conference on Machine Learning}.\hskip 1em plus 0.5em
  minus 0.4em\relax PMLR, 2017, pp. 4120--4129.

\bibitem{friedman2001elements}
J.~Friedman, T.~Hastie, and R.~Tibshirani, \emph{The elements of statistical
  learning}.\hskip 1em plus 0.5em minus 0.4em\relax Springer series in
  statistics New York, 2001, vol.~1, no.~10.

\bibitem{liu2020accelerating}
W.~Liu, L.~Chen, Y.~Chen, and W.~Zhang, ``Accelerating federated learning via
  momentum gradient descent,'' \emph{IEEE Transactions on Parallel and
  Distributed Systems}, vol.~31, no.~8, pp. 1754--1766, 2020.

\bibitem{huo2020faster}
Z.~Huo, Q.~Yang, B.~Gu, L.~C. Huang \emph{et~al.}, ``Faster on-device training
  using new federated momentum algorithm,'' \emph{arXiv preprint
  arXiv:2002.02090}, 2020.

\bibitem{lian2015asynchronous}
X.~Lian, Y.~Huang, Y.~Li, and J.~Liu, ``Asynchronous parallel stochastic
  gradient for nonconvex optimization,'' in \emph{Advances in Neural
  Information Processing Systems}, 2015, pp. 2737--2745.

\bibitem{lee2016gradient}
J.~D. Lee, M.~Simchowitz, M.~I. Jordan, and B.~Recht, ``Gradient descent
  converges to minimizers,'' \emph{arXiv preprint arXiv:1602.04915}, 2016.

\bibitem{lecun2010mnist}
Y.~LeCun, C.~Cortes, and C.~Burges, ``Mnist handwritten digit database,''
  \emph{ATT Labs [Online]. Available: http://yann.lecun.com/exdb/mnist},
  vol.~2, 2010.

\bibitem{xiao2017fashion}
H.~Xiao, K.~Rasul, and R.~Vollgraf, ``Fashion-mnist: a novel image dataset for
  benchmarking machine learning algorithms,'' \emph{arXiv preprint
  arXiv:1708.07747}, 2017.

\bibitem{cohen2017emnist}
G.~Cohen, S.~Afshar, J.~Tapson, and A.~Van~Schaik, ``Emnist: Extending mnist to
  handwritten letters,'' in \emph{2017 International Joint Conference on Neural
  Networks (IJCNN)}.\hskip 1em plus 0.5em minus 0.4em\relax IEEE, 2017, pp.
  2921--2926.

\bibitem{krizhevsky2009learning}
A.~Krizhevsky, G.~Hinton \emph{et~al.}, ``Learning multiple layers of features
  from tiny images,'' 2009.

\bibitem{merity2016pointer}
S.~Merity, C.~Xiong, J.~Bradbury, and R.~Socher, ``Pointer sentinel mixture
  models,'' \emph{arXiv preprint arXiv:1609.07843}, 2016.

\bibitem{vaswani2017attention}
A.~Vaswani, N.~Shazeer, N.~Parmar, J.~Uszkoreit, L.~Jones, A.~N. Gomez,
  {\L}.~Kaiser, and I.~Polosukhin, ``Attention is all you need,'' in
  \emph{Advances in neural information processing systems}, 2017, pp.
  5998--6008.

\bibitem{he2016deep}
K.~He, X.~Zhang, S.~Ren, and J.~Sun, ``Deep residual learning for image
  recognition,'' in \emph{Proceedings of the IEEE conference on computer vision
  and pattern recognition}, 2016, pp. 770--778.

\bibitem{zhu2019deep}
L.~Zhu, Z.~Liu, and S.~Han, ``Deep leakage from gradients,'' in \emph{Advances
  in Neural Information Processing Systems}, 2019, pp. 14\,774--14\,784.

\bibitem{zhao2020idlg}
B.~Zhao, K.~R. Mopuri, and H.~Bilen, ``idlg: Improved deep leakage from
  gradients,'' \emph{arXiv preprint arXiv:2001.02610}, 2020.

\bibitem{2020arXiv200302133L}
L.~{Lyu}, H.~{Yu}, and Q.~{Yang}, ``{Threats to Federated Learning: A
  Survey},'' \emph{arXiv e-prints}, p. arXiv:2003.02133, Mar. 2020.

\end{thebibliography}

\begin{IEEEbiography}[{\includegraphics[width=1in,height=1.25in,clip,keepaspectratio]{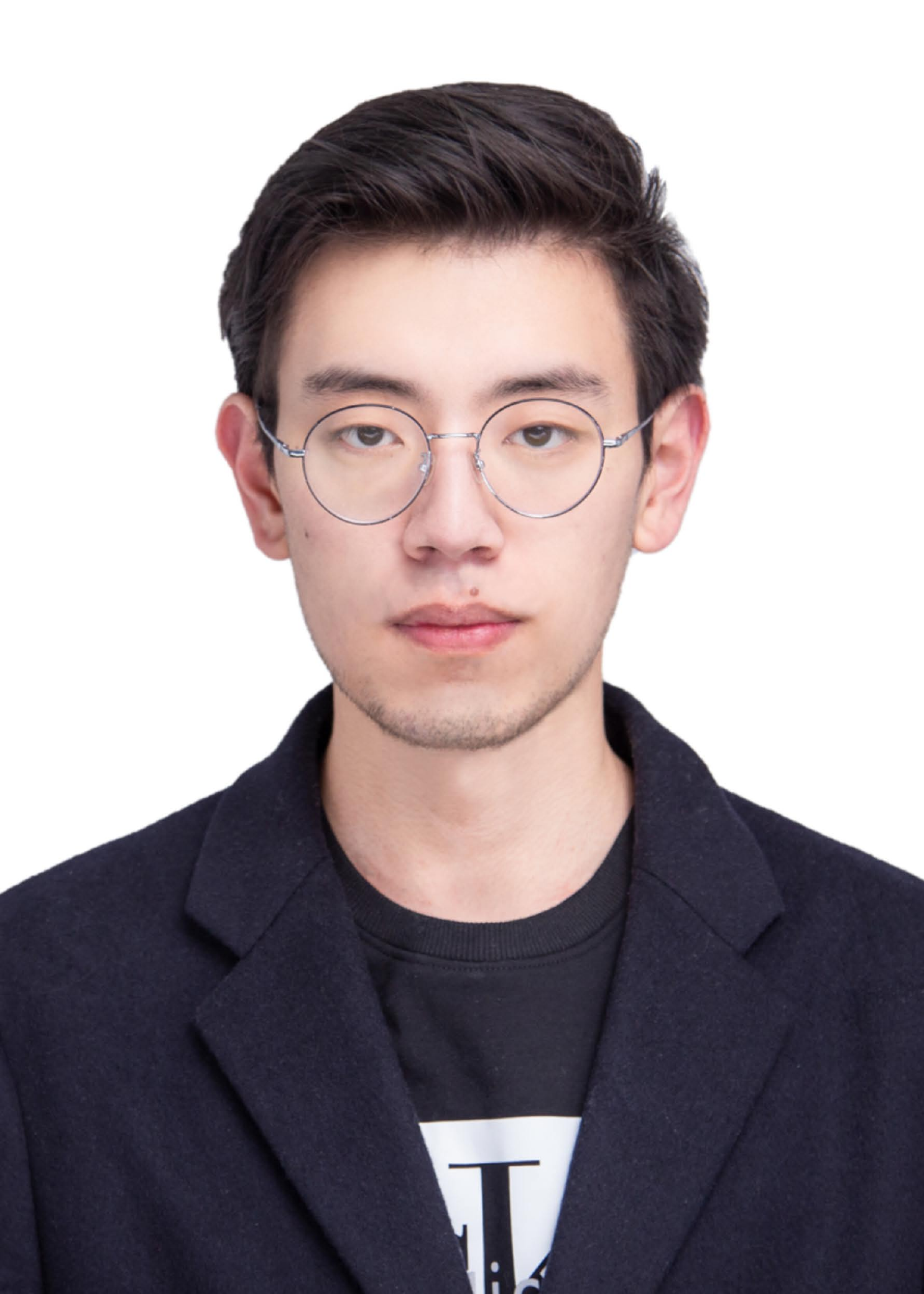}}]{Yuhao Zhou} is currently a senior student studying Computer Science in the College of Computer Science, Sichuan University, China. He has received his research training under Jiancheng Lv instructions since 2019. His main research interests include distributed machine learning, federated learning and optimization.
\end{IEEEbiography}

\begin{IEEEbiography}[{\includegraphics[width=1in,height=1.25in,clip,keepaspectratio]{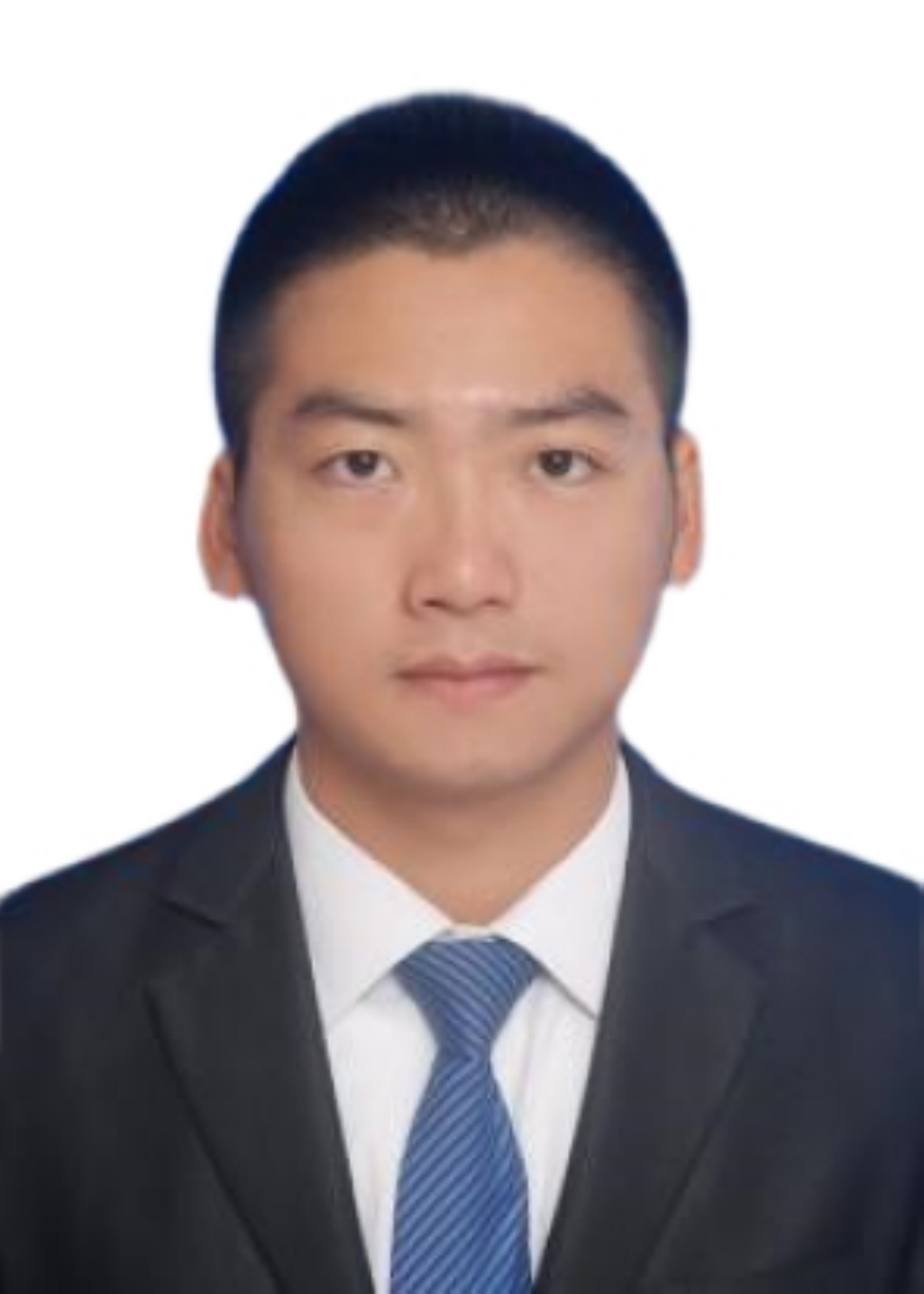}}]{Qing Ye} received the BS and MS degree from the School of Computing, Sichuan University, China. He is currently working toward the PhD degree with the School of Computing, Sichuan University. His main research interests include distributed machine learning, deep learning and neural architecture search.
\end{IEEEbiography}

\begin{IEEEbiography}[{\includegraphics[width=1in,height=1.25in,clip,keepaspectratio]{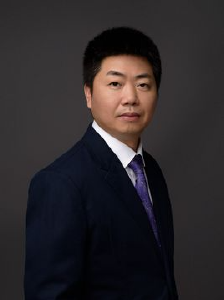}}]{Jiancheng Lv} (M’09) received the Ph.D. degree in computer science and engineering from the University of Electronic Science and Technology of China, Chengdu, China, in 2006.
 
He was a Research Fellow with the Department of Electrical and Computer Engineering, National University of Singapore, Singapore. He is currently a Professor with the Data Intelligence and Computing Art Laboratory, College of Computer Science, Sichuan University, Chengdu, China. His research interests include neural networks, machine learning, and big data.
\end{IEEEbiography}


\vfill


\end{document}